\documentclass[letterpaper]{article} 
\usepackage{aaai2026}   

\usepackage{times}     
\usepackage{helvet}    
\usepackage{courier}   
\usepackage[hyphens]{url}   
\usepackage{graphicx}         
\usepackage{natbib}           
\usepackage{caption}          
\usepackage[ruled,vlined,linesnumbered]{algorithm2e}

\usepackage{listings}
\usepackage{amsmath, amssymb, amsthm, amsfonts, mathtools, bm, array}
\usepackage{mathrsfs}
\usepackage{thmtools, thm-restate}
\usepackage{cleveref}
\usepackage{enumerate}

\usepackage{subcaption}

\usepackage[colorinlistoftodos,prependcaption,textsize=tiny]{todonotes}

\crefname{lemma}{Lemma}{Lemmas}
\crefname{theorem}{Theorem}{Theorems}
\crefname{corollary}{Corollary}{Corollaries}
\crefname{proposition}{Proposition}{Propositions}
\crefname{definition}{Definition}{Definitions}
\crefname{claim}{Claim}{Claims}
\crefname{remark}{Remark}{Remarks}

\newcommand{\kushagra}[1]{\todo[color=yellow, inline]{kushagra: #1}}
\newcommand{\tienlong}[1]{\todo[color=green, inline]{long: #1}}

\newtheorem{theorem}{Theorem}
\newtheorem{lemma}{Lemma}
\newtheorem{claim}{Claim}

\newtheorem{corollary}{Corollary}

\theoremstyle{definition}
\newtheorem{definition}{Definition}

\theoremstyle{remark}
\newtheorem{remark}[lemma]{Remark}
\newtheorem*{remark*}{Remark}

\renewcommand{\epsilon}{\varepsilon}

\newcommand{\m}[1]{\mathcal{#1}}

\newcommand{\n}{\nonumber}

\DeclareMathOperator{\dist}{dist}

\renewcommand{\emptyset}{\varnothing}

\newcommand{\fair}{\textit{Fair Clustering}}
\newcommand{\E}[1]{\operatorname{EXT}(#1)}
\newcommand{\I}[1]{\operatorname{INT}(#1)}
\newcommand{\cost}[1]{\operatorname{cost}(#1)}
\newcommand{\pay}[1]{\operatorname{pay}(#1)}
\newcommand{\fptwo}{\texttt{fairpower-of-two}}
\newcommand{\fequi}{\texttt{fair-equi}}
\newcommand{\greedymerge}{\texttt{multi-gm}}
\newcommand{\fcc}{\texttt{fairifyCC}}
\newcommand{\fmulti}{\texttt{make-pdc-fair}}
\newcommand{\fgen}{\texttt{fair-general}}

\newcommand{\costone}{\operatorname{cost}_1}
\newcommand{\costtwo}{\operatorname{cost}_2}
\newcommand{\costthree}{\operatorname{cost}_3}
\newcommand{\costfour}{\operatorname{cost}_4}
\newcommand{\costfive}{\operatorname{cost}_5}
\newcommand{\costsix}{\operatorname{cost}_6}
\newcommand{\costseven}{\operatorname{cost}_7}
\newcommand{\costeight}{\operatorname{cost}_8}
\newcommand{\costnine}{\operatorname{cost}_9}
\newcommand{\cutcase}{\texttt{Cut-Case}(c_j)}
\newcommand{\mergecase}{\texttt{Merge-Case}(c_j)}

\newcommand{\Sn}{T_a^k}
\newcommand{\Sx}{S_j^k}

\newcommand{\cut}{\texttt{CUT}}
\newcommand{\merge}{\texttt{MERGE}}
\newcommand{\cb}{\texttt{CB}}
\newcommand{\mb}{\texttt{MB}}

\newcommand{\card}[1]{\left\lvert #1 \right\rvert}
\newcommand{\cls}[1]{\m{#1}}

\newcommand{\clsfitting}[1][]{\mathrm{ClusterFitting}\ifx#1\empty\else(#1)\fi}

\newcommand{\threeclsf}{3\textsc{-Closest EquiFair}}
\newcommand{\clsf}[1]{#1\textsc{-Closest EquiFair}}

\newcommand{\np}{\mathbf{NP}}
\newcommand{\npc}{\mathbf{NP}\text{-complete}}
\newcommand{\nph}{\mathbf{NP}\text{-hard}}
\newcommand{\thrp}{3\textsc{-Partition}}

\newcommand{\pdca}{\texttt{create-pdc}}
\newcommand{\optcl}[1]{\m{F}^*_{#1}}
\newcommand{\outfptwo}{\m{F}_{\text{fpt}}}
\newcommand{\outmpf}{\m{F}_{\text{mpf}}}

% ...

\title{Generalizing Fair Clustering to Multiple Groups: Algorithms and Applications}

\author{
    Diptarka Chakraborty\textsuperscript{\rm 1}, Kushagra Chatterjee\textsuperscript{\rm 1}, Debarati Das\textsuperscript{\rm 2}, Tien-Long Nguyen\textsuperscript{\rm 2}\\
}
\affiliations{
    \textsuperscript{\rm 1} Nationaly University of Singapore \\
    \textsuperscript{\rm 2} Pennsylvania State University \\

    diptarka@comp.nus.edu.sg, kushagra.chatterjee@u.nus.edu, \{dxd5606,tfn5179\}@psu.edu
}

\begin{document}

\maketitle

\begin{abstract}
Clustering is a fundamental task in machine learning and data analysis, but it frequently fails to provide fair representation for various marginalized communities defined by multiple protected attributes -- a shortcoming often caused by biases in the training data. As a result, there is a growing need to enhance the fairness of clustering outcomes, ideally by making minimal modifications, possibly as a post-processing step after conventional clustering. Recently, Chakraborty et al. [COLT'25] initiated the study of \emph{closest fair clustering}, though in a restricted scenario where data points belong to only two groups. In practice, however, data points are typically characterized by many groups, reflecting diverse protected attributes such as age, ethnicity, gender, etc.

In this work, we generalize the study of the \emph{closest fair clustering} problem to settings with an arbitrary number (more than two) of groups. We begin by showing that the problem is NP-hard even when all groups are of equal size -- a stark contrast with the two-group case, for which an exact algorithm exists. Next, we propose near-linear time approximation algorithms that efficiently handle arbitrary-sized multiple groups, thereby answering an open question posed by Chakraborty et al. [COLT'25].

Leveraging our closest fair clustering algorithms, we further achieve improved approximation guarantees for the \emph{fair correlation clustering} problem, advancing the state-of-the-art results established by Ahmadian et al. [AISTATS'20] and Ahmadi et al. [2020]. Additionally, we are the first to provide approximation algorithms for the \emph{fair consensus clustering} problem involving multiple (more than two) groups, thus addressing another open direction highlighted by Chakraborty et al. [COLT'25].
\end{abstract}


\section{Introduction}

Clustering, the task of partitioning a set of data points into groups based on their mutual similarity or dissimilarity, stands as a fundamental problem in unsupervised learning and is ubiquitous in applications of machine learning and data analysis. Often, each data point carries certain protected attributes, which can be encoded by assigning a specific color to each point. While traditional clustering algorithms typically succeed in optimizing their target objectives, they frequently fail to ensure \emph{fairness} in their results. This shortfall can introduce or perpetuate biases against marginalized groups defined by sensitive attributes such as gender or race~\cite{kay2015unequal, bolukbasi2016man}. These potential biases arise not necessarily from the algorithms themselves, but from historical marginalization inherent in the data used for training. Addressing and mitigating such biases to achieve fair outcomes has emerged as a central topic in the field, with considerable attention given in recent years to the development of algorithms that guarantee \emph{demographic parity}~\cite{dwork2012fairness} and/or \emph{equal opportunity}~\cite{hardt2016equality}.

In the context of clustering, \cite{chierichetti2017fair} initiated the study of \emph{fair clustering} to address the issue of disparate impact and promote fair representation. Their work initially focused on datasets with two groups, each point colored either red or blue, and aimed to partition the data such that the blue to red ratio in every cluster matched that of the overall dataset. However, restricting attention to only two colors is limiting, as real-world data often involves multiple (and sometimes non-binary) protected attributes, such as age, race, or gender, resulting in several disjoint colored groups. Subsequent research generalized the fair clustering framework to accommodate more than two colors~\cite{rosner2018privacy}, requiring that the proportion of each colored group within clusters reflects the global proportions of colors. Further studies have considered scenarios where each color group is of equal size~\cite{bohm2020fair}. We refer to the related works for different variants of the clustering problems that have been studied under fairness constraints. 

As previously highlighted, a range of effective clustering algorithms are available for various clustering paradigms; however, these methods may yield unfair or biased results when the training data itself is biased. Such skewed clustering outcomes may lead to inequitable treatment, particularly if the clusters serve as the basis for decision-making or analysis. To counteract this, post-processing existing clustering solutions to mitigate bias and achieve fairness is often necessary, ideally with only minimal adjustments to the cluster assignments. Despite the fundamental nature of this problem, it has received limited attention within the broader context of clustering in prior research. However, it has been studied for specific metric spaces, such as ranking~\cite{CelisSV18, chakraborty2022, kliachkin2024fairness}. Only recently~\cite{chakraborty2025towards} did introduce the problem of obtaining a \emph{closest fair clustering} where given an existing clustering, finding a fair clustering by altering as few assignments as possible. Their study, however, was confined to the special case of two colored groups, which, as previously noted, is restrictive in many practical scenarios. In their work,~\cite{chakraborty2025towards} presented $O(1)$-approximation algorithms for the case of blue and red groups with arbitrary ratios, and demonstrated that the problem can be solved exactly in near-linear time when the groups are of equal size. In the end, they posed the general case of more than two colored groups as an intriguing open direction. In this paper, we address this open problem by devising approximation algorithms and establishing computational hardness that holds even for equi-proportioned multiple-colored groups, thereby showing a stark distinction from the two-group case.

Building upon our findings for the closest fair clustering problem, we further investigate their implications for other prominent clustering variants, specifically \emph{correlation clustering} and \emph{consensus clustering}. In correlation clustering, one is presented with a labeled (typically complete) graph where each edge is marked as either $+$ or $-$. The cost of a clustering is calculated as the total number of $+$ edges that span across clusters and $-$ edges that fall within clusters. This problem has widespread applications in fields such as data mining, social network analysis, computational biology, and marketing analysis~\cite{bonchi2014correlation, hou2016new, veldt2018correlation, bressan2019correlation, kushagra2019semi}. The \emph{fair correlation clustering} problem, introduced in~\cite{pmlr-v108-ahmadian20a, ahmadi2020fair}, seeks a fair clustering that minimizes this cost. In their work, approximation algorithms were proposed for cases involving multiple colored groups, with the approximation factor depending on the ratios of the group sizes. In this paper, we improved upon their approximation bound and are the first to achieve an approximation guarantee independent of the group size ratio.

In consensus clustering, the objective is to derive a single, representative (consensus) clustering from a collection of clusterings over the same set of data points, to minimize a chosen objective function. The objective often depends on the application, with the most common being the \emph{median} that minimizes the total distance to all input clusterings and the \emph{center}  that minimizes the maximum distance. Here, the distance between two clusterings is typically defined by the number of point pairs that are co-clustered (together) in one clustering but not in the other. Consensus clustering is widely applicable across various fields, including bioinformatics~\cite{filkov2004integrating, filkov2004heterogeneous}, data mining~\cite{topchy2003combining}, and community detection~\cite{lancichinetti2012consensus}. 
Recently,~\cite{chakraborty2025towards} extended the study of consensus clustering to incorporate fairness constraints, providing constant-factor approximations but only in the special case of two colored groups. In this work, we expand upon these results by establishing approximation guarantees for consensus clustering problems involving more than two colored groups.

\subsection{Our Contribution}

In this work, we present both new algorithms and hardness results for the \emph{Closest Fair Clustering} problem in settings where the input data points come from multiple groups, and also study two well-known applications, namely, the \emph{Fair Correlation Clustering} and \emph{Fair Consensus Clustering} problems, and provide new algorithm results for both. Below, we summarize our main contributions.

 \paragraph{Closest Fair Clustering with Multiple Groups.}
 In this problem, we are given a clustering $\m{D} = \{ D_1, D_2, \ldots, D_m\}$ defined on a dataset $V$ where $V$ is classified into a set $\chi$ of disjoint groups, each represented by a unique color. 
The objective is to compute a fair clustering of $V$, where the proportion of data points from each group (or color) within any output cluster should reflect their overall proportions in the entire dataset, while also minimizing the distance to the input clustering $\mathcal{D}$.



\begin{itemize}
\item Considering $|\chi|$ to be the total number of colors, our first algorithm handles the case where each output cluster must contain the different colored points according to global ratio $p_1:p_2: \cdots : p_{|\chi|}$, and achieves an $O(|\chi|^{3.81})$-approximation.

This study significantly generalizes the work of~\cite{chakraborty2025towards} [COLT '25], which focused solely on the binary (two-group) setting.

\item Next, we consider the special case where each output cluster must contain an equal number of data points from each color group, and we provide an $O(|\chi|^{1.6} \log^{2.81} |\chi|)$-approximation for the Closest Fair Clustering problem.
Furthermore, when $|\chi|$ is a power of two, we present an improved algorithm with an $O(|\chi|^{1.6})$-approximation. 


\item Finally, we show that the problem is NP-hard for any setting involving more than two colors, even when all color groups are equally represented in an output cluster. 

This shows a clear hardness gap between the two-color setting, where an exact algorithm exists, and the multi-color case, where the problem becomes computationally intractable; thus underscoring the necessity of our approximation algorithms.
 \end{itemize}

\paragraph{Application to Fair Correlation Clustering.}

Building on the above result, we study their implications for another key clustering variant: the \emph{Fair Correlation Clustering} problem. In correlation clustering, the input is a graph with edges labeled as + or -, and the goal is to produce a clustering minimizing disagreements; + edges between clusters and - edges within clusters. In the fair version, the clustering must also satisfy group fairness constraints. This problem has been shown to be NP-hard~\cite{ahmadi2020fair}.

\begin{itemize}

\item We begin by designing an algorithm for the setting where each output cluster must contain points from different color groups as per the global ratio $p_1:p_2: \cdots : p_{|\chi|}$, achieving an $O(|\chi|^{3.81})$-approximation.

This eliminates the dependence on the max-min color ratio $q = \frac{\max(p_j)}{\min(p_j)}$ that appeared in the previous $O(q^2 |\chi|^2)$ bound of~\cite{pmlr-v108-ahmadian20a, ahmadi2020fair}, where $q$ can be as large as a polynomial in $|V|$. 

\item For the special case where each output cluster must contain an equal number of data points from each color group, we improve the approximation to $O(|\chi|^{1.6} \log^{2.81} |\chi|)$, and further to $O(|\chi|^{1.6})$ when $|\chi|$ is a power of two. 

This improves upon the previous $O(|\chi|^2)$ bound given by~\cite{pmlr-v108-ahmadian20a, ahmadi2020fair}.

    \end{itemize}

\paragraph{Application to Fair Consensus Clustering.}

Next, we turn to another application: the \emph{consensus clustering} problem. The goal here is to compute a single representative (consensus) clustering from a collection of input clusterings over the same dataset, minimizing a specified objective function (e.g., median or center) while also satisfying group fairness constraints. By combining a triangle inequality argument with our results for Closest Fair Clustering, we obtain new approximation guarantees for this problem. These results generalize the work of~\cite{chakraborty2025towards}, which was limited to the binary (two-group) setting, to the more general multi-group case.


\begin{itemize}

\item Analogous to our results for fair correlation clustering, we obtain an $O(|\chi|^{3.81})$-approximation algorithm for the general case with arbitrary group proportions.

        \item For the equi-proportion case, we again improve the approximation to $O(|\chi|^{1.6} \log^{2.81} |\chi|)$, and further to $O(|\chi|^{1.6})$ when $|\chi|$ is a power of two.

    \end{itemize}
The details are provided in the appendix.

\subsection{Related Works}
Since the introduction of the fair clustering in~\cite{chierichetti2017fair}, recent years have witnessed a significant increase in research focused on different aspects of fair clustering problems. The literature so far encompasses numerous variants of the fair clustering problem with multiple colored groups, such as $k$-center/median/means clustering~\cite{chierichetti2017fair, HuangJV19}, scalable clustering~\cite{BackursIOSVW19}, proportional clustering~\cite{ChenFLM19}, fair representational clustering~\cite{bera2019fair, bercea2019cost}, pairwise fair clustering~\cite{bandyapadhyay2024coresets, bandyapadhyay2024polynomial, bandyapadhyay2025constant}, correlation clustering~\cite{pmlr-v108-ahmadian20a, ahmadi2020fair, ahmadian2023improved}, 1-clustering over rankings~\cite{wei22, chakraborty2022, chakraborty2025improved}, and consensus clustering~\cite{chakraborty2025towards}, among others.

A comprehensive study of the correlation clustering problem was first undertaken in~\cite{bansal2004correlation}. Since then, correlation clustering has been studied across various graph settings, including the extensively examined complete graphs~\cite{ailon2008aggregating, chawla2015near} and weighted graphs~\cite{demaine2006correlation}. The problem, even when restricted to complete graphs, is known to be \texttt{APX}-hard~\cite{charikar2005clustering}, with the best-known approximation algorithm currently achieving a $1.438$-approximation factor~\cite{cao2024understanding}. Its fair variant remains \texttt{NP}-hard even in the case of two color groups of equal sizes~\cite{ahmadi2020fair}, and several approximation algorithms have been developed for both the exact fairness notion~\cite{pmlr-v108-ahmadian20a, ahmadi2020fair} and the relaxed fairness notion~\cite{ahmadian2023improved}. 

The consensus clustering problem, under both the median and center objectives, is known to be \texttt{NP}-hard~\cite{kvrivanek1986np, swamy2004correlation} and in fact, \texttt{APX}-hard (that is it is unlikely to have an $(1+\epsilon)$-factor algorithm for any $\epsilon >0$) even with as few as three input clusterings~\cite{BonizzoniVDJ08}. Currently, the best-known algorithms include an $11/7$-approximation for the median objective~\cite{ailon2008aggregating} and an approximation slightly better than 2 for the center objective~\cite{DK25}. Apart from that, various heuristics have been proposed to produce reasonable solutions (e.g.,\cite{goder2008consensus, monti2003consensus, wu2014k}). More recently,\cite{chakraborty2025towards} began examining fair consensus clustering, focusing on only two colored groups.

\section{Preliminaries}

In this section, we define key terms and concepts that are essential for understanding our proofs and algorithms.

\begin{definition}[$\fair$]
    Given a set of points $V$ and a set of colors $\chi = \{c_1, c_2, \ldots,c_k\}$, suppose $c_j(V) \subseteq V$ be the set of points of color $c_j$ in $V$. We call a clustering $\m{F}$ of $V$ a $\fair$ if for all clusters $F \in \m{F}$ we have
    \[
         \card{ c_{1}(F) } : \cdots: \card{ c_{k}(F) } = \card{c_{1}(V)} :\cdots: \card{c_{k}(V)} .
    \]
\end{definition}

For two clustering $\m{C}$ and $\m{C}'$ of $V$ we define $\dist(\m{C}, \m{C'})$ as the distance between two clustering $\m{C}$ and $\m{C}'$. The distance is measured by the number of pairs $(u,v)$ that are together in $\m{C}$ but separated by $\m{C}'$ and the number of pairs $(u,v)$ that are separated by $\m{C}$ but together in $\m{C}'$. More specifically,
\begin{align*}
\dist(\m{C}, \m{C}') = 
\big| \big\{ \{u,v\} \mid\,& u,v \in V,\; 
[u \sim_{\m{C}} v \land u \not\sim_{\m{C}'} v] \\
&\lor [u \not\sim_{\m{C}} v \land u \sim_{\m{C}'} v] \big\} \big|
\end{align*}
where $u \sim_{\m{C}} v$ denotes whether both $u$ and $v$ belong to the same cluster in $\m{C}$ or not.

\begin{definition}[Closest $\fair$]
    Given an arbitrary clustering $\m{D}$, a clustering $\optcl{\m{D}}$ is called a closest $\fair$ to $\m{D}$ if for all $\fair$ $\m{F}$ we have $\dist(\m{D}, \m{F}) \geq \dist(\m{D}, \optcl{\m{D}})$. 
\end{definition}

We denote a closest $\fair$ to $\m{D}$ by the notation $\optcl{\m{D}}$.

\paragraph{$\gamma$-close $\fair$} We call a $\fair$ $\m{F}$ a $\gamma$-close $\fair$ to a clustering $\m{D}$ if
\[
    \dist(\m{D}, \m{F}) \leq \gamma \dist(\m{D}, \optcl{\m{D}}).
\]

\section{Approximate Closest Fair Clustering for Equi-Proportion Groups}
\label{sec1}

In this section, we provide an approximation algorithm to find a closest {\fair} when all the groups are of equal size.

\begin{theorem} \label{thm:equitheorem}
    There exists an algorithm that, given a clustering $\m{D}$ where each color group contains an equal number of points, computes a $O(|\chi|^{1.6} \log^{2.81} |\chi|)$-close $\fair$ in $O(|V| \log |V|)$ time, where $\chi$ denotes the set of colors. Moreover, when $|\chi|$ is a power of two, the algorithm computes a $O(|\chi|^{1.6})$-close $\fair$.
\end{theorem}

First, let us handle the case when $|\chi|$ is a power of $2$. To do that, we provide an algorithm $\fptwo$, which produces an $O(|\chi|^{1.6})$-close fair clustering $\outfptwo$ to a clustering $\m{D}$ when $|\chi|$ is a power of $2$.

\paragraph{Overview of the Algorithm \(\fptwo\):}
Let the input be a clustering \( \mathcal{D} = \{ D_1, D_2, \ldots, D_m \} \), where each point is colored from a color set \( \chi = \{ c_1, \ldots, c_k \} \), and assume \( |\chi| \) is a power of two. The goal is to output a clustering \( \outfptwo \) in which every cluster contains an equal number of points of each color, i.e., \( c_p(F_a) = c_q(F_a) \) for all \( p \neq q \), $F_a \in \outfptwo$.

The algorithm proceeds in \( \log |\chi| \) iterations. At iteration \( i \), the color set \( \chi \) is partitioned into \( |\chi| / 2^i \) disjoint blocks of size \( 2^i \), defined as:
\[
B_j^i = \{ c_{(j-1)\cdot 2^i + 1}, \ldots, c_{j \cdot 2^i} \}, \quad \text{for } j = 1, \ldots, |\chi| / 2^i.
\]
Let \( \mathcal{N}^i \) be the clustering at iteration \( i \), with \( \mathcal{N}^0 := \mathcal{D} \). The algorithm maintains the invariant that, in every cluster \( N_a^i \in \mathcal{N}^i \), the colors within each block \( B_j^i \) appear equally.

\paragraph{Surplus Definition:}
For adjacent blocks \( B_j^i \) and \( B_{j+1}^i \) in a cluster \( N_a^i \), the surplus \( T_a^j \) is the excess of the larger bucket over the smaller. The surplus is chosen so that all colors in the surplus are equally represented.

\paragraph{Algorithm \(\fptwo\):}
\begin{enumerate}
    \item \textbf{Initialization:} Set \( \mathcal{N}^0 \gets \mathcal{D} \).
    
    \item \textbf{Iterative Refinement:} For each iteration \( i = 1 \) to \( \log |\chi| \):
    \begin{itemize}
        \item Initialize \( \mathcal{N}^i \gets \mathcal{N}^{i-1} \).
        \item For each pair of disjoint adjacent blocks \( (B_j^i, B_{j+1}^i) \):
        \begin{itemize}
            \item For each cluster \( N_a^i \in \mathcal{N}^i \) and odd $j$, compute the surplus \( T_a^j \) between adjacent blocks $B_j^i(N_a^i)$ and $B_{j+1}^i(N_a^i)$, and remove it. Here for a color block $B_j^i$, $B_j^i(N_a^i)$ is the set of points in $N_a^i$ that has a color from the block $B_j^i$.
            \item Store removed surpluses into sets \( S_j \) or \( S_{j+1} \), depending on which block had the surplus.
            \item Call \texttt{multi-GM}\((S_j, S_{j+1})\) to form new fair clusters and add them to \( \mathcal{N}^i \).
        \end{itemize}
    \end{itemize}
    
    \item \textbf{Output:} Return \( \mathcal{N}^{\log |\chi|} \rightarrow \outfptwo \).
\end{enumerate}

\paragraph{Subroutine \texttt{multi-GM}:}
Given two collections of point groups from blocks \( B_j^i \) and \( B_{j+1}^i \), the \texttt{multi-GM} procedure greedily merges pairs of subsets into fair subsets in which each color from \( B_j^i \cup B_{j+1}^i \) is equally represented.

The subroutine:
\begin{itemize}
    \item Iteratively selects one subset from each collection.
    \item Trims the larger subset to match the smaller, preserving equal color counts.
    \item Merges the trimmed subsets into a fair set and adds it to the output.
\end{itemize}
Continue this until no further fair subsets can be formed.

We provide the pseudocode of the algorithms $\fptwo$ and $\greedymerge$ in the appendix.

\subsection{Proof of \cref{thm:equitheorem}}
In this section, we analyze the algorithm $\fptwo$ by first establishing \cref{lem:fair-power-of-two} stated below,

\begin{lemma}\label{lem:fair-power-of-two}
    Given a clustering $\m{D}$ as input, the algorithm $\fptwo$ computes a $O(|\chi|^{1.6})$-close $\fair$, where $|\chi|$ is a power of $2$.
\end{lemma}

We show the above result by generating a sequence of $\log |\chi|$ intermediate clusterings, with each intermediate step involving an approximation factor of 2, and thus finally achieving a factor of $3^{\log |\chi|}=|\chi|^{1.6}$. We provide the proof in the appendix.

The above \cref{lem:fair-power-of-two} proves \cref{thm:equitheorem} when $|\chi|$ is a power of $2$. Now, we prove \cref{thm:equitheorem} for any values of $|\chi|$. To do that, we need to describe the algorithm $\fmulti$.

\paragraph{Overview of the Algorithm \(\fmulti\):}
Given a clustering \( \mathcal{I} = \{ I_1, I_2, \ldots, I_s \} \) over a point set \( V \), where each point is colored from \( \zeta = \{ z_1, \ldots, z_r \} \) and satisfies the global proportion:
\[
z_1(V) : z_2(V) : \cdots : z_r(V) = p_1 : p_2 : \cdots : p_r,
\]
the goal is to construct a fair clustering \( \outmpf \) such that every cluster \( F \in \outmpf \) satisfies this ratio. We assume w.l.o.g. that \( p_1 > p_2 > \cdots > p_r \). Also assume that each input cluster is \emph{\(p\)-divisible}, i.e., \( z_j(I_i) \) is a multiple of \( p_j \). We call such clustering as \emph{$p$-divisble clustering}(\texttt{pdc} for short).

\paragraph{Algorithm $\fmulti$:}
\begin{enumerate}[1.]
\item The algorithm proceeds in \( T = \lceil \log_2 r \rceil \) iterations. Let \( \mathcal{F}^0 := \mathcal{I} \), and define:
\[
\m{I} = \mathcal{F}^0 \rightarrow \mathcal{F}^1 \rightarrow \cdots \rightarrow \mathcal{F}^T = \outmpf,
\]
where each \( \mathcal{F}^t \) enforces proportionality within blocks of colors.

\item \textbf{Hierarchical Block Structure:}
At iteration \( t \), the color set is partitioned into blocks \( \{ B_1^t, B_2^t, \ldots, B_{m_t}^t \} \), constructed hierarchically:
\[
B_i^t = B_{2i-1}^{t-1} \cup B_{2i}^{t-1}, \quad \text{with singleton blocks } B_j^0 = \{ z_j \}.
\]
If \( m_{t-1} \) is odd, the last block is carried forward unchanged.

\item \textbf{Balancing Rule:}
To merge two sub-blocks \( A = B_{2i-1}^{t-1} \) and \( B = B_{2i}^{t-1} \), consider a cluster \( F^{t-1} \in \mathcal{F}^{t-1} \), where
\[
z_c(F^{t-1}) = p_c \cdot x \text{ for } z_c \in A, \quad z_d(F^{t-1}) = p_d \cdot y \text{ for } z_d \in B.
\]
To equalize the scaling factors \( x \) and \( y \), we do:
\begin{itemize}
    \item If \( x > y \): merge \( p_d \cdot (x - y) \) points of color \( z_d \in B \) into \( F^{t-1} \).
    \item If \( x < y \): cut \( p_d \cdot (y - x) \) points of color \( z_d \in B \) from \( F^{t-1} \).
\end{itemize}

This ensures that the merged block \( B_i^t = A \cup B \) in each cluster \( F^t \in \mathcal{F}^t \) satisfies the combined proportionality.

\item \textbf{Output:}
After \( T = \lceil \log_2 r \rceil \) iterations, the final clustering \( \outmpf \) satisfies: For all $F \in \outmpf$,
\begin{align*}
   z_1(F) : z_2(F) : \cdots : z_r(F) = p_1 : p_2 : \cdots : p_r . 
\end{align*}
\end{enumerate}

We provide the pseudocode of $\fmulti$ in the appendix.

To analyse the algorithm $\fmulti$ we need to prove the following lemma.


\begin{lemma} \label{lem:analyze-fmulti}
    The algorithm \(\fmulti\) outputs a clustering \(\mathcal{F}\) that is \( O( r^{2.81}) \)-close $\fair$ to the input clustering \(\mathcal{I} \), where $r$ is the number of colors.
\end{lemma}
We show the above result by generating a sequence of $\log r$ intermediate clusterings, with each intermediate step involving an approximation factor of 6, and thus finally achieving a factor of $7^{\log r}=r^{2.81}$. We provide a detailed proof in the appendix.

Using \cref{lem:fair-power-of-two} and \cref{lem:analyze-fmulti}, we can prove \cref{thm:equitheorem}. We provide the full proof of \cref{thm:equitheorem} in the appendix.

\begin{proof}[Proof Sketch of \cref{thm:equitheorem}]
We design an algorithm \(\fequi\) to convert an arbitrary clustering \(\mathcal{D}\), where the color set \(\chi\) is not necessarily a power of two, into a fair clustering \(\mathcal{F}\) in which every color appears equally in each cluster. The algorithm proceeds in two main stages:

\begin{enumerate}
    \item \textbf{Color Grouping and Intermediate Fairness:}
    \begin{itemize}
        \item Partition the color set \(\chi\) into \(\log |\chi|\) disjoint groups \(G_1, \ldots, G_{\log |\chi|}\), where each group’s size is a power of two. This is done greedily by assigning group sizes according to the binary representation of \(|\chi|\).
        \item Apply the algorithm \(\fptwo\) to obtain an intermediate clustering \(\mathcal{I}\), in which each group \(G_\ell\) satisfies intra-group fairness: all colors in \(G_\ell\) appear equally in each cluster.
    \end{itemize}
    
    \item \textbf{Global Fairness via Multi-Group Merging:}
    \begin{itemize}
        \item Treat each group \(G_\ell\) as a single meta-color and apply the algorithm \(\fmulti\) on \(\mathcal{I}\) to obtain the final clustering \(\mathcal{F}\).
        \item \(\fmulti\) ensures that across all clusters, the meta-colors (i.e., groups) are in proportion to their sizes, and also restores uniformity within each group, thus achieving full color-wise fairness.
    \end{itemize}
\end{enumerate}

\textbf{Approximation Bound:} 
\begin{itemize}
    \item By \cref{lem:fair-power-of-two}, the intermediate clustering \(\mathcal{I}\) is \(O(|\chi|^{1.6})\)-close to \(\mathcal{D}\).
    \item By \cref{lem:analyze-fmulti}, the final clustering \(\mathcal{F}\) is \(O(\log^{2.81}|\chi|)\)-close to \(\mathcal{I}\).
    \item Combining via triangle inequality yields:
    \[
    \dist(\mathcal{D}, \mathcal{F}) \leq O(|\chi|^{1.6} \log^{2.81}|\chi|) \cdot \dist(\mathcal{D}, \optcl{\m{D}})
    \]
    where \(\optcl{\m{D}}\) is the closest fair clustering to \(\mathcal{D}\).
\end{itemize}

This completes the proof sketch.
\end{proof}


\section{Approximate Closest Fair Clustering for Arbitrary-Proportion Groups}
\label{sec2}

In this section, we prove the following theorem.

\begin{theorem} \label{thm:arbitrary-proportion}
    There is an algorithm that, given an arbitrary clustering $\m{D}$ over a vertex set $V$ where each vertex $v \in V$ has a color in $\chi = \{ c_1, \ldots, c_k\}$, finds a $O(|\chi|^{3.81})$-close $\fair$ $\m{F}$ in time $O(|V| \log |V|)$.
\end{theorem}


To prove the above theorem, we take a $2$-step approach similar to \cite{chakraborty2025towards}, which was constrained to only two colors.

\begin{enumerate}[(i)]
    \item First, we convert the clustering $\m{D}$ to a clustering $\m{M}$ such that for each cluster $M_i \in \m{M}$, $c_j(M_i)$ is divisible by $p_j$. Recall we call such a clustering a $p$-divisible clustering.
    \item In the second step, we will provide the clustering $\m{M}$ as input to the algorithm $\fmulti$ and get a fair clustering $\m{F}$ as output. 
\end{enumerate}

Now, we provide an algorithm $\pdca$ to convert a clustering $\m{D}$ on a vertex set $V$ to a $p$-divisible clustering $\m{M}$.

\paragraph{Overview of the algorithm $\pdca$:}

\textbf{Input:} A clustering \( \mathcal{D} = \{ D_1, \ldots, D_m \} \) over vertex set \( V \), color set \( \chi = \{ c_1, \ldots, c_k \} \), and a proportion vector \( \mathbf{p} = (p_1, \ldots, p_k) \) satisfying:
\[
c_1(V) : c_2(V) : \cdots : c_k(V) = p_1 : p_2 : \cdots : p_k.
\]

\textbf{Goal:} Convert \( \mathcal{D} \) into a \( p \)-divisible clustering \( \mathcal{M} \) where each cluster contains a multiple of \( p_j \) points of color \( c_j \).

\paragraph{Key Definitions:}
\begin{itemize}
    \item Surplus: \( \sigma(D_i, c_j) \subseteq D_i \) of size \( c_j(D_i) \bmod p_j \) if $p_j \nmid c_j(D_i)$, else it has size $p_j$, denoting excess \( c_j \)-colored vertices in \( D_i \).
    \item Total surplus: \( \sigma_j = \sum_{D_i \in \mathcal{D}} |\sigma(D_i, c_j)| \) (always a multiple of \( p_j \)).
    \item Deficit: \( \delta(D_i, c_j) \subseteq V \setminus D_i \), of size \( p_j - |\sigma(D_i, c_j)| \). Represents the number of $c_j$ colored points required to make it a multiple of $p_j$.
    \item Cut and merge costs:
    \begin{align*}
        &\kappa^j(D_i) = |\sigma(D_i,c_j)| \cdot (|D_i| - |\sigma(D_i,c_j)|) \n \\
    &\mu^j(D_i) = |\delta(D_i,c_j)| \cdot |D_i|
    \end{align*}
    
\end{itemize}

\paragraph{Algorithm $\pdca$:}
\begin{enumerate}
    \item For each color \( c_j \), initialize \( \sigma_j / p_j \) empty auxiliary clusters \( \{ P_1, \ldots, P_{\sigma_j/p_j} \} \).

    \item Classify each cluster \( D_i \in \mathcal{D} \) into:
    \begin{itemize}
        \item \texttt{CUT}, if \( |\sigma(D_i, c_j)| \leq p_j/2 \)
        \item \texttt{MERGE}, otherwise.
    \end{itemize}

    \item \textbf{Cut and redistribute:} While \texttt{CUT} is non-empty:
    \begin{itemize}
        \item For \( D_i \in \texttt{CUT} \), remove \( \sigma(D_i, c_j) \) from $D_k$.
        \item Try to donate surplus to deficits in \( D_\ell \in \texttt{MERGE} \).
        \item If no deficit remains, assign surplus to an available extra cluster \( P_m \), ensuring each \( P_m \) reaches size \( p_j \).
    \end{itemize}

    \item \textbf{Handle remaining merges:} While deficits remain:
    \begin{itemize}
        \item Pick the cluster with minimal \( \kappa^j(D_k) - \mu^j(D_k) \), redistribute its surplus as above. Note that in this step, we can pick a cluster multiple times.
        \item Remove a cluster $D_\ell \in \merge$ if it deficit is satisfied.
        \item Repeat until all deficits are filled.
    \end{itemize}

    \item \textbf{Output:} The final clustering \( \mathcal{M} \) where each cluster is \( p \)-divisible for all colors.
\end{enumerate}

We provide the pseudocode of the algorithm $\pdca$ in the appendix.

Now, we analyze the algorithm $\pdca$. To analyze and state the lemma, let us define some terms

\begin{itemize}
    \item \textbf{Optimal-$p$-divisible clustering}: Given a clustering $\m{D}$, we call a clustering $\m{M}^*$ an optimal-$p$-divisible clustering if $\dist(\m{D}, \m{M}^*) \leq \dist(\m{D}, \m{M})$. 
    
    \item \textbf{$\alpha$-close-$p$-divisible clustering}: Given a clustering $\m{D}$, we call a clustering $\m{M}$ an $\alpha$-close-$p$-divisible clustering if $
        \dist(\m{D}, \m{M}) \leq \alpha \dist(\m{D}, \m{M}^*)$.
\end{itemize}

Now, we state the lemma for analysing the algorithm $\pdca$.

\begin{lemma}\label{lem:main-multiple-of-p}
    Given a clustering $\m{D}$, the algorithm $\pdca$ outputs a clustering $\m{M}$ such that $\m{M}$ is $O(|\chi|)$-close-$p$-divisible clustering to $\m{D}$.
\end{lemma}

We provide the proof of \cref{lem:main-multiple-of-p} in the appendix. Next, we prove \cref{thm:arbitrary-proportion} assuming \cref{lem:main-multiple-of-p}.

\begin{proof}[Proof of \cref{thm:arbitrary-proportion}]
To prove the theorem, we describe an algorithm $\fgen$ that proceeds in two stages. First, we apply the algorithm $\pdca$ to the input clustering $\m{D}$ to obtain a $p$-divisible clustering $\m{M}$ that is $O(|\chi|)$-close to $\m{D}$. Then, we apply the algorithm $\fmulti$ on $\m{M}$ to obtain the final fair clustering $\m{F}$, which is $O(|\chi|^{2.81})$-close to $\m{M}$ (see \cref{lem:analyze-fmulti}). By combining the guarantees from both steps, we conclude that $\m{F}$ is $O(|\chi|^{3.81})$-close $\fair$ to the input $\m{D}$.
    \begin{align}
        \dist(\m{D}, \m{F}) &\leq \dist(\m{D}, \m{M}) + \dist(\m{M}, \m{F}) \, \, \text{(triangle inequality)} \n \\
        &\leq O(|\chi|) \dist(\m{D}, \m{F}^*) + O(|\chi|^{2.81}) \dist(\m{M}, \m{F}^*) \n \\
        &\leq O(|\chi|) \dist(\m{D}, \m{F}^*) + O(|\chi|^{2.81}) (\dist(\m{D}, \m{M}) \n \\ &+ \dist(\m{D}, \m{F}^*)) \, \, \text{(triangle inequality)} \n \\
        &\leq O(|\chi| + |\chi|^{3.81} + |\chi|^{2.81}) \dist(\m{D}, \m{F}^*) \n \\
        &\leq O(|\chi|^{3.81}) \dist(\m{D}, \m{F}^*) \n
    \end{align}
    This completes the proof of \cref{thm:arbitrary-proportion}.
\end{proof}
\section{Hardness for Three Equi-Proportion Groups}
\label{sec:three_colors_closest_fair_clustering_is_hard_even_with_equiproportion} 

\newcommand{\hinpclss}{\mathcal{H}}
\newcommand{\hinpcls}[1][]{\cls{D}\ifx#1\empty\else_{#1}\fi}
\newcommand{\hfairoptclss}{\mathcal{F}}
\newcommand{\hfaircls}[1][]{F\ifx#1\empty\else_{#1}\fi}
\newcommand{\hsfairopt}{m} 
\newcommand{\cchi}{\card{\chi}} 
\newcommand{\hgbcls}[1]{\mathrm{GB}_{#1}}
\newcommand{\hbcls}[1]{\mathrm{B}_{#1}}
\newcommand{\hrcls}[1]{\mathrm{R}_{#1}}
\newcommand{\hmono}[1]{\chi(#1)}
\newcommand{\hclspfinpclss}{\hinpclss^{c}}
\newcommand{\hclspfoptclss}{\hfairoptclss^{c}}
\newcommand{\hclspfcls}{ F^{ c }}
\newcommand{\hyes}{\mathrm{YES}}
\newcommand{\hno}{\mathrm{NO}}

In this section, we show that finding a closest fair clustering to a given clustering where each point is assigned one color from a set of $ k\geq 3 $ colors is hard even when the number of points in each color class is equal. Our reduction also extends to arbitrary color ratios.


We begin by defining the decision version of closest fair clustering with multiple colors and equal representation.
\begin{definition}[$ \clsf{k} $]
    Given a clustering $ \hinpclss $ over a set of points $ V $ where each point is assigned with one of the colors from a set of $ k\geq 3 $ colors, and the numbers of points of each color are equal, together with a non-negative integer $ \tau $, decide between the following:
    \begin{itemize}
        \item $ \hyes $: There exists a {\fair} (on input point set) $ \hfairoptclss $ such that $\dist(\hinpclss, \hfairoptclss) \le \tau$;
        \item $ \hno $: For every {\fair} (on input point set) $ \hfairoptclss $, $\dist(\hinpclss, \hfairoptclss) > \tau$.
    \end{itemize}
\end{definition}

We show the following theorem.

\begin{theorem}\label{thm:three.closest.equi.fair.is.np.complete}
    For any integer $ k\geq 3 $, $ \clsf{k} $ is $ \np $-hard.
\end{theorem}

We present a polynomial-time reduction from the $ \thrp $ problem (defined below) to $ \clsf{k} $. 

\begin{definition}[$ \thrp $]\label{def:three_partition}
    Given a (multi)set of positive integers $S = \{ x_1, \ldots, x_d\}$, decide whether ($ \hyes $:) there exists a partition of $S$ into $m$ disjoint subsets $S_1, S_2, \ldots, S_f \subseteq S$ where $f = d/3$, such that
    \begin{itemize}
        \item For all $i$, $|S_i| = 3$; and
        \item For all $i$, $\sum_{x_j \in S_i}x_j = T$, where $T = \frac{\sum_{x_j \in S}x_j}{n/3}$,
    \end{itemize}
    or (NO:) no such partition exists.
\end{definition}

We apply our reduction from a more restricted version of $ \thrp $, in which each $x_i \in S$ satisfies $x_i \in (T/4, T/2)$, where $T = \frac{3}{n} \sum_{x_i \in S} x_i$, and additionally $ x_{i}\leq d^{b} $, for some non-negative constant $ b $. Note that, this variant remains $ \npc $, as established by~\cite{garey1975complexity}, which shows that $ \thrp $ as defined in~\cref{def:three_partition} is \emph{strongly $ \npc $}. Hence, for the rest of this section, we refer to this restricted version simply as $\thrp$. 



\paragraph{Reduction from $ \thrp $ to $ \clsf{k} $. } Given a $\thrp$ instance $S = \{ x_1, x_2, \ldots, x_d\}$ we create a $ \clsf{k} $ instance $( \hinpclss, \tau )$ as follows:

\begin{itemize}
    \item $ \hinpclss = \{ \hgbcls{1}, \hgbcls{2}, \dots, \hgbcls{d/3}, \hrcls{1}, \hrcls{2}, \dots, \hrcls{d} \} $, where for each $i \in \{1,\ldots,d/3\}$, $ \hgbcls{i} $ is a cluster of size $ (k-1)T $ with $ T $ points of color $ c_{t}\ (2\leq t\leq k) $, and for each $j \in \{1,\ldots, n\}$, $ \hrcls{j} $ is a monochromatic $ c_{1} $ cluster (i.e., containing only points with color $ c_{1} $) of size $ x_j $ (i.e., $ \card{\hrcls{j}} = x_j$);
    \item Set  $
            \tau = \dfrac{n}{3}(k-1)T^{2} + \dfrac{1}{2}\sum_{i=1}^{n}x_{i}(T-x_{i}),\text{ if } k> 3 $,
        and, $\tau = 2\sum_{i=1}^{n}x_{i}^{2} + 2\sum_{i=1}^{n}x_{i}(T-x_{i}),\text{ if } k=3$.
\end{itemize}

Note that each $ x_{i}\leq d^{b} $, for some non-negative constant $ b $, the size of the instance $( \hinpclss, \tau )$ is polynomial in $d$. Moreover, it is straightforward to see that the reduction runs in polynomial time.

The following lemma argues that the above reduction maps a $ \hyes $ instance of the $ \thrp $ to a $ \hyes $ instance of the $ \clsf{k} $.

\begin{restatable}{lemma}{yesinstance}\label{lem:yes.instance}
    For any integer $ k\geq 3 $, if $ S $ is a $ \hyes $ instance of the $\thrp$, then $(\hinpclss, \tau) $ is also a $ \hyes $ instance of the $ \clsf{k} $.
\end{restatable}
We defer the proof of~\cref{lem:yes.instance} to the appendix.

It remains to demonstrate that our reduction maps a $ \hno $ instance of the $\thrp$ to a $ \hno $ instance of the $ \clsf{k} $.

\begin{lemma}\label{lem:no.instance}
    For $ k \geq 3$, if $S $ is a $ \hno $ instance of the $ \thrp $, then $(\hinpclss, \tau) $ is also a $ \hno $ instance of the $ \clsf{k} $.
\end{lemma}

\begin{proof}
    Assume to the contrary that $ ( \hinpclss,\tau ) $ is a $ \hyes $ instance. Then there exists a $ \fair $ $ \hfairoptclss $ such that $ \dist( \hinpclss, \hfairoptclss )\leq \tau $. 

    Without loss of generality, we refer to $ c_{ 1 } $ as the color red, and refer to $ c_{ 2 } $ as the color blue.

    Let $ V' $ be the set of red-blue points obtained from $ V $ by recoloring every point with color $ c_{ j } $ ($ j\geq 3 $) to blue. Denote $ \hclspfinpclss $ and $ \hclspfoptclss $ the clusterings obtained from $ \hinpclss $ and $ \hfairoptclss $, respectively, under this recoloring. Then, $ \dist( \hclspfinpclss, \hclspfoptclss ) = \dist( \hinpclss, \hfairoptclss ) $. We analyze the structure of $ \hclspfinpclss $ and $ \hclspfoptclss $.

    Observe that $ \hclspfinpclss $ is a clustering over $ V' $, in which $ \hclspfinpclss = \{ \hbcls{ 1 }, \hbcls{ 2 }, \dots, \hbcls{ n/3 }, \hrcls{ 1 }, \hrcls{ 2 }, \dots, \hrcls{ n } \} $, where each $ \hbcls{ i } $ is a monochromatic blue cluster of size $ ( k-1 )T $, obtained from the original cluster $ \hgbcls{ i } $ in $ \hinpclss $ by recoloring every point to blue. Each $ \hrcls{ i } $ is a monochromatic red cluster, which remains unchanged from $ \hinpclss $. 

    Now we claim that $ \hclspfoptclss $ is a $ \fair $ clustering over $ V' $. Indeed, recall that $ \hfairoptclss $ is a $ \fair $ over $ V $. Hence, in every cluster $ \hfaircls\in \hfairoptclss $, the numbers of points of each color are equal, that is, $ \card{ c_{ i }( \hfaircls )} = \card{ c_{ j }( \hfaircls ) } $, for all $ 1\leq i, j\leq k $. Note that each cluster $ \hclspfcls\in \hclspfoptclss $ is obtained from a cluster in $ \hfairoptclss $ by recoloring every point with color $ c_{ i } $ ($ i\geq 3 $) to blue. Hence, the ratio between the number of blue points and the number of red points in $ \hclspfcls $ is $ ( k-1 ) $. This implies that $ \hclspfoptclss $ is a $ \fair $ over $ V' $.

    Applying a result from~\cite[Lemma 45, Lemma 46]{chakraborty2025towards} for the set of points $ V' $, it follows that $ \dist( \hclspfinpclss, \hclspfoptclss ) > \tau $, which is a contradiction since we have established that $ \dist( \hclspfinpclss, \hclspfoptclss ) = \dist( \hinpclss, \hfairoptclss ) \leq \tau $. This concludes that $ ( \hinpclss, \tau ) $ is a $ \hno $ instance.
\end{proof}

%

As our reduction runs in polynomial time, from~\cref{lem:yes.instance} and~\cref{lem:no.instance}, we conclude that $ \clsf{k} $ is $ \nph $. This completes the proof of~\cref{thm:three.closest.equi.fair.is.np.complete}. In the appendix, we remark on how to generalize this reduction for arbitrary ratios.

\section{Implication to Fair Correlation Clustering}

\paragraph{Correlation Clustering.} 
Given a complete undirected graph $G(V, E)$ where each edge $(u,v) \in E$ is labeled either ``$+$'' (similar) or ``$-$'' (dissimilar), let $E^+$ and $E^-$ denote the sets of ``$+$'' and ``$-$'' edges, respectively. A clustering $\m{C} = \{ C_1, C_2, \ldots, C_t \}$ partitions $V$ into disjoint subsets.

Let $\E{\m{C}}$ denote the set of inter-cluster edges and $\I{\m{C}}$ the set of intra-cluster edges. The cost of clustering $\m{C}$ is defined as:
\[
\cost{\m{C}} = |\E{\m{C}} \cap E^+| + |\I{\m{C}} \cap E^-|.
\]
The goal is to find a clustering that minimizes $\cost{\m{C}}$. In addition, when we want $\m{C}$ to be also a {\fair}, the problem is referred to as \emph{fair correlation clustering}.

\begin{theorem}\label{thm:correlation-clustering}
    There exists an algorithm that, given a correlation clustering instance $G$, computes a $O(|\chi|^{1.6} \log^{2.81} |\chi|)$ approximate fair correlation clustering when there are equal number of data points from each color group, and a $O(|\chi|^{3.81})$ approximate fair correlation clustering when the ratio between the number of points from different color groups is arbitrary.
\end{theorem}

To prove \cref{thm:correlation-clustering}, we use the following.

\begin{lemma}\label{lem:correlation_clustering}
Let $\mathcal{C}$ be a clustering, and suppose there exists an algorithm $\mathcal{A}$ that computes an $\gamma$-close fair clustering with respect to $\mathcal{C}$. Additionally, suppose there exists a $\beta$-approximation algorithm $\m{B}$, for the standard correlation clustering problem on a graph $G$. Then, there exists an algorithm that computes a fair correlation clustering of $G$ with approximation factor $(\gamma + \beta + \gamma\beta)$.
\end{lemma}

\begin{proof}
    Let us first describe the algorithm $\fcc$, which computes a fair correlation clustering for a given instance \( G \).

\begin{itemize}
    \item \textbf{Input:} Correlation clustering instance \( G \).
    \item \textbf{Output:} A fair clustering \( \mathcal{F} \).
    \item \textbf{Procedure:}
    \begin{enumerate}
        \item Compute a correlation clustering \( \mathcal{D} \) of \( G \) using a \( \beta \)-approximation algorithm \( \mathcal{B} \).
        \item Apply the closest fair clustering algorithm $\m{A}$ to \( \mathcal{D} \) to obtain a fair clustering \( \mathcal{F} \) that is \( \gamma \)-close to \( \mathcal{D} \).
        \item Return \( \mathcal{F} \).
    \end{enumerate}
\end{itemize}

We argue that $\m{F}$ is $(\gamma + \beta + \gamma\beta)$ approximate correlation clustering of $G$ using triangle inequality. We defer the proof to the appendix.
\end{proof}

 \begin{proof}[Proof of \cref{thm:correlation-clustering}]
     By~\cite{cao2025solving}, we get that there exists an $O(1)$-approximation algorithm to find a correlation clustering for a correlation clustering instance $G$. When each color group has the same size, the algorithm $\fequi$ produces an $O(|\chi|^{1.6} \log^{2.81} |\chi|)$-close clustering to any input clustering $\m{D}$. Hence, by \cref{lem:correlation_clustering} we get that the algorithm $\fcc$ produces an $O(|\chi|^{1.6} \log^{2.81} |\chi|)$ approximate fair correlation clustering.
     
    In the general case with arbitrary group size ratio $p_1:p_2:\cdots:p_{|\chi|}$, the algorithm $\fgen$ gives $O(|\chi|^{3.81})$-close clustering to any input clustering $\m{D}$. Hence, again by \cref{lem:correlation_clustering} we show algorithm $\fcc$ produces an $O(|\chi|^{3.81})$ approximate fair correlation clustering for this.
 \end{proof}

\section{Conclusion}
In this paper, we generalize the closest fair clustering problem originally proposed by~\cite{chakraborty2025towards} [COLT '25] to scenarios involving any number of groups, thereby addressing settings with non-binary, multiple protected attributes. We demonstrate that the problem becomes NP-hard even when there are just three equal-sized groups, showing a strong separation with the two equi-proportion group case where an exact solution exists. We further propose near-linear time approximation algorithms for clustering with multiple (potentially unequal-sized) groups, answering an open problem posed by~\cite{chakraborty2025towards} [COLT '25]. Leveraging these results, we achieve improved approximation guarantees for fair correlation clustering and, for the first time, provide approximation guarantees for fair consensus clustering involving more than two groups. Promising directions for future research include improving the approximation factors further and investigating alternative fairness criteria with similar approximation guarantees.

\section*{Acknowledgments}Diptarka Chakraborty was supported in part by an MoE AcRF Tier 1 grant (T1 251RES2303), and a Google South \& South-East Asia Research Award. Kushagra Chatterjee was supported by an MoE AcRF Tier 1 grant (T1 251RES2303). Debarati Das was supported in part by NSF grant 2337832. 

\bibliography{j_ref}

@article{charikar2005clustering,
  title={Clustering with qualitative information},
  author={Charikar, Moses and Guruswami, Venkatesan and Wirth, Anthony},
  journal={Journal of Computer and System Sciences},
  volume={71},
  number={3},
  pages={360--383},
  year={2005},
  publisher={Elsevier}
}

@article{bansal2004correlation,
  title={Correlation clustering},
  author={Bansal, Nikhil and Blum, Avrim and Chawla, Shuchi},
  journal={Machine learning},
  volume={56},
  number={1},
  pages={89--113},
  year={2004},
  publisher={Springer}
}

@inproceedings{chawla2015near,
  title={Near optimal lp rounding algorithm for correlationclustering on complete and complete k-partite graphs},
  author={Chawla, Shuchi and Makarychev, Konstantin and Schramm, Tselil and Yaroslavtsev, Grigory},
  booktitle={Proceedings of the forty-seventh annual ACM symposium on Theory of computing},
  pages={219--228},
  year={2015}
}

@article{demaine2006correlation,
  title={Correlation clustering in general weighted graphs},
  author={Demaine, Erik D and Emanuel, Dotan and Fiat, Amos and Immorlica, Nicole},
  journal={Theoretical Computer Science},
  volume={361},
  number={2-3},
  pages={172--187},
  year={2006},
  publisher={Elsevier}
}

@inproceedings{cao2025solving,
  title={Solving the Correlation Cluster LP in Sublinear Time},
  author={Cao, Nairen and Cohen-Addad, Vincent and Lee, Euiwoong and Li, Shi and Lolck, David Rasmussen and Newman, Alantha and Thorup, Mikkel and Vogl, Lukas and Yan, Shuyi and Zhang, Hanwen},
  booktitle={Proceedings of the 57th Annual ACM Symposium on Theory of Computing},
  pages={1154--1165},
  year={2025}
}

@inproceedings{bonchi2014correlation,
  title={Correlation clustering: from theory to practice.},
  author={Bonchi, Francesco and Garcia-Soriano, David and Liberty, Edo},
  booktitle={KDD},
  pages={1972},
  year={2014}
}

@article{hou2016new,
  title={A new correlation clustering method for cancer mutation analysis},
  author={Hou, Jack P and Emad, Amin and Puleo, Gregory J and Ma, Jian and Milenkovic, Olgica},
  journal={Bioinformatics},
  volume={32},
  number={24},
  pages={3717--3728},
  year={2016},
  publisher={Oxford University Press}
}

@article{bressan2019correlation,
  title={Correlation clustering with adaptive similarity queries},
  author={Bressan, Marco and Cesa-Bianchi, Nicol{\`o} and Paudice, Andrea and Vitale, Fabio},
  journal={Advances in neural information processing systems},
  volume={32},
  year={2019}
}

@inproceedings{kushagra2019semi,
  title={Semi-supervised clustering for de-duplication},
  author={Kushagra, Shrinu and Ben-David, Shai and Ilyas, Ihab},
  booktitle={The 22nd International Conference on Artificial Intelligence and Statistics},
  pages={1659--1667},
  year={2019},
  organization={PMLR}
}

@inproceedings{veldt2018correlation,
  title={A correlation clustering framework for community detection},
  author={Veldt, Nate and Gleich, David F and Wirth, Anthony},
  booktitle={Proceedings of the 2018 World Wide Web Conference},
  pages={439--448},
  year={2018}
}

@inproceedings{rosner2018privacy,
  title={Privacy Preserving Clustering with Constraints},
  author={R{\"o}sner, Clemens and Schmidt, Melanie},
  booktitle={45th International Colloquium on Automata, Languages, and Programming (ICALP 2018)},
  pages={96--1},
  year={2018},
  organization={Schloss Dagstuhl--Leibniz-Zentrum f{\"u}r Informatik}
}

@article{bohm2020fair,
  title={Fair clustering with multiple colors},
  author={B{\"o}hm, Matteo and Fazzone, Adriano and Leonardi, Stefano and Schwiegelshohn, Chris},
  journal={arXiv preprint arXiv:2002.07892},
  year={2020}
}

@article{bera2019fair,
  title={Fair algorithms for clustering},
  author={Bera, Suman and Chakrabarty, Deeparnab and Flores, Nicolas and Negahbani, Maryam},
  journal={Advances in Neural Information Processing Systems},
  volume={32},
  year={2019}
}

@inproceedings{bercea2019cost,
  title={On the Cost of Essentially Fair Clusterings},
  author={Bercea, Ioana O and Gro{\ss}, Martin and Khuller, Samir and Kumar, Aounon and R{\"o}sner, Clemens and Schmidt, Daniel R and Schmidt, Melanie},
  booktitle={Approximation, Randomization, and Combinatorial Optimization. Algorithms and Techniques (APPROX/RANDOM 2019)},
  pages={18--1},
  year={2019},
  organization={Schloss Dagstuhl--Leibniz-Zentrum f{\"u}r Informatik}
}

@article{bandyapadhyay2024coresets,
  title={On coresets for fair clustering in metric and euclidean spaces and their applications},
  author={Bandyapadhyay, Sayan and Fomin, Fedor V and Simonov, Kirill},
  journal={Journal of Computer and System Sciences},
  volume={142},
  pages={103506},
  year={2024},
  publisher={Elsevier}
}

@article{bandyapadhyay2024polynomial,
  title={A Polynomial-Time Approximation for Pairwise Fair $ k $-Median Clustering},
  author={Bandyapadhyay, Sayan and Chlamt{\'a}{\v{c}}, Eden and Friggstad, Zachary and Jamshidian, Mahya and Makarychev, Yury and Vakilian, Ali},
  journal={arXiv preprint arXiv:2405.10378},
  year={2024}
}

@inproceedings{bandyapadhyay2025constant,
  title={A Constant-Factor Approximation for Pairwise Fair k-Center Clustering},
  author={Bandyapadhyay, Sayan and Chen, Tianzhi and Friggstad, Zachary and Jamshidian, Mahya},
  booktitle={International Conference on Integer Programming and Combinatorial Optimization},
  pages={43--57},
  year={2025},
  organization={Springer}
}

@article{chakraborty2025improved,
  title={Improved Rank Aggregation under Fairness Constraint},
  author={Chakraborty, Diptarka and Das, Himika and Dey, Sanjana and Yan, Alvin Hong Yao},
  journal={arXiv preprint arXiv:2505.10006},
  year={2025}
}

@article{chakraborty2025towards,
  title={Towards Fair Representation: Clustering and Consensus},
  author={Chakraborty, Diptarka and Chatterjee, Kushagra and Das, Debarati and Nguyen, Tien Long and Nobahari, Romina},
  journal={38th Annual Conference on Learning Theory (COLT 2025); full version: arXiv preprint arXiv:2506.08673},
  year={2025}
}

@inproceedings{wei22,
  author    = {Dong Wei and
               Md Mouinul Islam and
               Baruch Schieber and
               Senjuti Basu Roy},
  title     = {Rank Aggregation with Proportionate Fairness},
  booktitle = {{SIGMOD} International Conference on Management of Data},
  pages     = {262--275},
  year      = {2022},
}

@inproceedings{chierichetti2017fair,
  title={Fair clustering through fairlets},
  author={Chierichetti, Flavio and Kumar, Ravi and Lattanzi, Silvio and Vassilvitskii, Sergei},
  booktitle={Advances in Neural Information Processing Systems ({NeurIPS})},
  pages={5029-5037},
  year={2017}
}

@inproceedings{BackursIOSVW19,
  author    = {Arturs Backurs and
               Piotr Indyk and
               Krzysztof Onak and
               Baruch Schieber and
               Ali Vakilian and
               Tal Wagner},
  title     = {Scalable Fair Clustering},
  booktitle = {International Conference on Machine Learning (ICML)},
  volume    = {97},
  pages     = {405--413},
  year      = {2019}
}

@inproceedings{ChenFLM19,
  author    = {Xingyu Chen and
               Brandon Fain and
               Liang Lyu and
               Kamesh Munagala},
  title     = {Proportionally Fair Clustering},
  booktitle = {International Conference on Machine Learning (ICML)},
  pages     = {1032--1041},
  year      = {2019}
}

@inproceedings{HuangJV19,
  author    = {Lingxiao Huang and
               Shaofeng H.{-}C. Jiang and
               Nisheeth K. Vishnoi},
  title     = {Coresets for Clustering with Fairness Constraints},
  booktitle = {Advances in Neural Information Processing Systems (NeurIPS)},
  pages     = {7587--7598},
  year      = {2019}
}

@inproceedings{dwork2012fairness,
  title={Fairness through awareness},
  author={Dwork, Cynthia and Hardt, Moritz and Pitassi, Toniann and Reingold, Omer and Zemel, Richard},
  booktitle={Innovations in Theoretical Computer Science},
  pages={214--226},
  year={2012}
}

@inproceedings{hardt2016equality,
  title={Equality of opportunity in supervised learning},
  author={Hardt, Moritz and Price, Eric and Srebro, Nati},
  booktitle={Advances in Neural Information Processing Systems (NeurIPS)},
  pages = {3315--3323},
  year={2016}
}

@inproceedings{CelisSV18,
  author    = {L Elisa Celis and
               Damian Straszak and
               Nisheeth K Vishnoi},
  title     = {Ranking with Fairness Constraints},
  booktitle = {International Colloquium on Automata, Languages, and Programming (ICALP)},
  volume    = {107},
  pages     = {28:1--28:15},
  year      = {2018}
}

@InProceedings{pmlr-v108-ahmadian20a,
  title = 	 {Fair Correlation Clustering},
  author =       {Ahmadian, Sara and Epasto, Alessandro and Kumar, Ravi and Mahdian, Mohammad},
  booktitle = 	 {International Conference on Artificial Intelligence and Statistics (AISTATS)},
  pages = 	 {4195--4205},
  year = 	 {2020},
  volume = 	 {108}
}

@inproceedings{kay2015unequal,
  title={Unequal representation and gender stereotypes in image search results for occupations},
  author={Kay, Matthew and Matuszek, Cynthia and Munson, Sean A},
  booktitle={{ACM} conference on human factors in computing systems},
  pages={3819--3828},
  year={2015}
}

@article{bolukbasi2016man,
  title={Man is to computer programmer as woman is to homemaker? debiasing word embeddings},
  author={Bolukbasi, Tolga and Chang, Kai-Wei and Zou, James Y and Saligrama, Venkatesh and Kalai, Adam T},
  journal={Advances in Neural Information Processing Systems (NeurIPS)},
  volume={29},
  year={2016}
}

@inproceedings{kliachkin2024fairness,
  title={Fairness in Ranking: Robustness through Randomization without the Protected Attribute},
  author={Kliachkin, Andrii and Psaroudaki, Eleni and Mare{\v{c}}ek, Jakub and Fotakis, Dimitris},
  booktitle={2024 IEEE 40th International Conference on Data Engineering Workshops (ICDEW)},
  pages={201--208},
  year={2024},
  organization={IEEE}
}

@article{chakraborty2022,
  title={Fair rank aggregation},
  author={Chakraborty, Diptarka and Das, Syamantak and Khan, Arindam and Subramanian, Aditya},
  journal={Advances in Neural Information Processing Systems},
  volume={35},
  pages={23965--23978},
  year={2022},
  note = {Full version: arXiv preprint arXiv:2308.10499}
}

@inproceedings{cao2024understanding,
  title={Understanding the Cluster Linear Program for Correlation Clustering},
  author={Cao, Nairen and Cohen-Addad, Vincent and Lee, Euiwoong and Li, Shi and Newman, Alantha and Vogl, Lukas},
  booktitle={Proceedings of the 56th Annual ACM Symposium on Theory of Computing},
  pages={1605--1616},
  year={2024}
}

@inproceedings{ahmadian2023improved,
  title={Improved approximation for fair correlation clustering},
  author={Ahmadian, Sara and Negahbani, Maryam},
  booktitle={International Conference on Artificial Intelligence and Statistics},
  pages={9499--9516},
  year={2023},
  organization={PMLR}
}

@article{ahmadi2020fair,
  title={Fair correlation clustering},
  author={Ahmadi, Saba and Galhotra, Sainyam and Saha, Barna and Schwartz, Roy},
  journal={arXiv preprint arXiv:2002.03508},
  year={2020}
}

@article{lancichinetti2012consensus,
  title={Consensus clustering in complex networks},
  author={Lancichinetti, Andrea and Fortunato, Santo},
  journal={Scientific reports},
  volume={2},
  number={1},
  pages={336},
  year={2012},
  publisher={Nature Publishing Group UK London}
}

@inproceedings{goder2008consensus,
  title={Consensus clustering algorithms: Comparison and refinement},
  author={Goder, Andrey and Filkov, Vladimir},
  booktitle={2008 Proceedings of the Tenth Workshop on Algorithm Engineering and Experiments (ALENEX)},
  pages={109--117},
  year={2008},
  organization={SIAM}
}

@article{monti2003consensus,
  title={Consensus clustering: a resampling-based method for class discovery and visualization of gene expression microarray data},
  author={Monti, Stefano and Tamayo, Pablo and Mesirov, Jill and Golub, Todd},
  journal={Machine learning},
  volume={52},
  pages={91--118},
  year={2003},
  publisher={Springer}
}

@article{wu2014k,
  title={K-means-based consensus clustering: A unified view},
  author={Wu, Junjie and Liu, Hongfu and Xiong, Hui and Cao, Jie and Chen, Jian},
  journal={IEEE transactions on knowledge and data engineering},
  volume={27},
  number={1},
  pages={155--169},
  year={2014},
  publisher={IEEE}
}

@article{BonizzoniVDJ08,
  author       = {Paola Bonizzoni and
                  Gianluca Della Vedova and
                  Riccardo Dondi and
                  Tao Jiang},
  title        = {On the Approximation of Correlation Clustering and Consensus Clustering},
  journal      = {J. Comput. Syst. Sci.},
  volume       = {74},
  number       = {5},
  pages        = {671--696},
  year         = {2008}
}

@article{filkov2004integrating,
  title={Integrating microarray data by consensus clustering},
  author={Filkov, Vladimir and Skiena, Steven},
  journal={International Journal on Artificial Intelligence Tools},
  volume={13},
  number={04},
  pages={863--880},
  year={2004},
  publisher={World Scientific}
}

@inproceedings{filkov2004heterogeneous,
  title={Heterogeneous data integration with the consensus clustering formalism},
  author={Filkov, Vladimir and Skiena, Steven},
  booktitle={International Workshop on Data Integration in the Life Sciences},
  pages={110--123},
  year={2004},
  organization={Springer}
}

@inproceedings{topchy2003combining,
  title={Combining multiple weak clusterings},
  author={Topchy, Alexander and Jain, Anil K and Punch, William},
  booktitle={Third IEEE international conference on data mining},
  pages={331--338},
  year={2003},
  organization={IEEE}
}

@article{kvrivanek1986np,
  title={NP-hard problems in hierarchical-tree clustering},
  author={K{\v{r}}iv{\'a}nek, Mirko and Mor{\'a}vek, Jaroslav},
  journal={Acta informatica},
  volume={23},
  pages={311--323},
  year={1986},
  publisher={Springer}
}

@inproceedings{swamy2004correlation,
  title={Correlation Clustering: maximizing agreements via semidefinite programming.},
  author={Swamy, Chaitanya},
  booktitle={SODA},
  volume={4},
  pages={526--527},
  year={2004},
  organization={Citeseer}
}

@article{ailon2008aggregating,
  title={Aggregating inconsistent information: ranking and clustering},
  author={Ailon, Nir and Charikar, Moses and Newman, Alantha},
  journal={Journal of the ACM (JACM)},
  volume={55},
  number={5},
  pages={1--27},
  year={2008},
  publisher={ACM New York, NY, USA}
}

@inproceedings{DK25,
  author       = {Debarati Das and
                  Amit Kumar},
  editor       = {Yossi Azar and
                  Debmalya Panigrahi},
  title        = {Breaking the Two Approximation Barrier for Various Consensus Clustering
                  Problems},
  booktitle    = {Proceedings of the 2025 Annual {ACM-SIAM} Symposium on Discrete Algorithms,
                  {SODA} 2025, New Orleans, LA, USA, January 12-15, 2025},
  pages        = {323--372},
  publisher    = {{SIAM}},
  year         = {2025}
}

@article{garey1975complexity,
  title={Complexity results for multiprocessor scheduling under resource constraints},
  author={Garey, Michael R and Johnson, David S.},
  journal={SIAM journal on Computing},
  volume={4},
  number={4},
  pages={397--411},
  year={1975},
  publisher={SIAM}
}

\clearpage

\newpage

\appendix

\section{Closest Fair Clustering for Equi-Proportion: Missing Details}

\begin{algorithm}[t]
\DontPrintSemicolon
\KwIn{Clustering $\mathcal{D}$}
\KwOut{A fair clustering $\m{N}$}
Let, $\m{N}^0 = \{ N^0_1, N^0_2, \ldots, N^0_\ell \} = \m{D}$. \;
\For{$i \gets 1$ \KwTo $\log |\chi|$}{
    $\m{N}^i = \m{N}^{i - 1}$ \;
    $j \gets 1$ \;
    \While{$j \neq |\chi| / 2^i$}{
        $ S_{j}, S_{j+1} \gets \emptyset $\;
        \ForEach{$N^{i}_a \in \mathcal{N}^{i}$}{
            $N^i_a \gets N^{i}_a \setminus T^j_a$\; 
            \If{$|B_j^i(N^{i}_a)| \geq |B_{j+1}^i(N^{i}_a)|$}{
                $S_j \gets S_{j} \cup T^j_a$\;
            }
            \Else{
                $S_{j+1} \gets S_{j+1} \cup T^j_a$\;
            }
        }
        $\m{N}^{i} = \m{N}^i \cup \texttt{multi-GM}(S_j, S_{j+1})$\;
        $j \gets j + 2^i$.
    }
}
\Return{$\m{N}^{\log |\chi|}$}
\caption{$\fptwo (\m{D})$}\label{algo:fair-power-of-two}
\end{algorithm}

\begin{algorithm}[t]
\DontPrintSemicolon
\KwIn{Two sets \texttt{Set1}, \texttt{Set2} of subsets of $V$, from blocks $B_j^i$ and $B_{j+1}^i$ respectively}
\KwOut{A set \texttt{Fair} of fair merged vertex sets}

\texttt{Fair} $\gets \emptyset$\;

\While{\texttt{Set1} $\neq \emptyset$ \textbf{and} \texttt{Set2} $\neq \emptyset$}{
    Let $S_1 \in \texttt{Set1}$\;
    Let $S_2 \in \texttt{Set2}$\;

    \If{$|S_1| \geq |S_2|$}{
        Let $S \subseteq S_1$ such that $|S| = |S_2|$ and $S$ contains equal number of vertices of each color in $B_j^i$.\;
        \texttt{Fair} $\gets$ \texttt{Fair} $\cup \{ S \cup S_2 \}$\;
        $S_1 \gets S_1 \setminus S$\;
        \texttt{Set2} $\gets$ \texttt{Set2} $\setminus \{S_2\}$\;
        \If{$S_1 = \emptyset$}{
            \texttt{Set1} $\gets$ \texttt{Set1} $\setminus \{S_1\}$\;
        }
    }
    \Else{
        Let $S \subseteq S_2$ such that $|S| = |S_1|$ and $S$ contains equal number of vertices of each color in $B_{j+1}^i$\;
        \texttt{Fair} $\gets$ \texttt{Fair} $\cup \{ S \cup S_1 \}$\;
        $S_2 \gets S_2 \setminus S$\;
        \texttt{Set1} $\gets$ \texttt{Set1} $\setminus \{S_1\}$\;
        \If{$S_2 = \emptyset$}{
            \texttt{Set2} $\gets$ \texttt{Set2} $\setminus \{S_2\}$\;
        }
    }
}
\Return{\texttt{Fair}}\;
\caption{\texttt{multi-GM}$(\texttt{Set1}, \texttt{Set2})$} \label{algo:greedymerge}
\end{algorithm}

\begin{algorithm}[t]
\DontPrintSemicolon
\KwIn{Initial clustering \( \mathcal{I} = \mathcal{F}^0 \), color set \( \zeta = \{ z_1, \ldots, z_r \} \), target proportions \( (p_1 : \cdots : p_r) \) \\ For each cluster $I_j \in \m{I}$ we have $z_k(I_j)$ is a multiple of $p_k$.}
\KwOut{Clustering \( \mathcal{F} \) where each cluster satisfies the global color ratio}

Let \( T \gets \lceil \log_2 r \rceil \)\;
Initialize blocks \( B^0_j \gets \{ z_j \} \) for all \( j \in [r] \)\;
Set \( \mathcal{F}^0 \gets \mathcal{I} \)\;

\For(\tcp*[f]{Iterate through block levels}){$t \gets 1$ \KwTo $T$}{
    Merge adjacent blocks from level \( t - 1 \) to form \( \{ B^t_1, B^t_2, \ldots \} \)\;

    Let \( m_{t-1} \gets \text{number of blocks at iteration } t - 1 \)\;
    
    \If{$m_{t-1}$ is odd}{
        Copy the last block as-is: \( B^t_{\lceil \#B^{t-1} / 2 \rceil} \gets B^{t-1}_{m_{t-1}} \)\;
    }
    
    Initialize \( \mathcal{F}^t \gets \emptyset \)\;
    
    \ForEach(\tcp*[f]{Iterate through clusters}){$F \in \mathcal{F}^{t-1}$}{
        \ForEach(\tcp*[f]{Iterate through blocks}){$B^t_i = B^{t-1}_{2i-1} \cup B^{t-1}_{2i}$}{
            Let \( A = B^{t-1}_{2i-1} = \{ z_{a_1}, \ldots, z_{a_s} \} \)\;
            Let \( B = B^{t-1}_{2i} = \{ z_{b_1}, \ldots, z_{b_u} \} \)\;
            
            Compute scaling factors:\;
            \Indp
                \( x \gets \min_{j \in [s]} \left( \frac{z_{a_j}(F)}{p_{a_j}} \right) \)\;
                \( y \gets \min_{k \in [u]} \left( \frac{z_{b_k}(F)}{p_{b_k}} \right) \)\;
            \Indm
            
            \uIf(\tcp*[f]{Case 1: need to merge}){$x > y$}{
                \ForEach{$k \in [u]$}{
                    Merge \( p_{b_k} \cdot (x - y) \) vertices of color \( z_{b_k} \) into \( F \)\;
                }
            }
            \uElseIf(\tcp*[f]{Case 2: need to cut}){$x < y$}{
                \ForEach{$k \in [u]$}{
                    Cut \( p_{b_k} \cdot (y - x) \) vertices of color \( z_{b_k} \) from \( F \)\;
                }
            }
            
            Add updated cluster \( F \) to \( \mathcal{F}^t \)\;
        }
    }
}

\Return{\( \mathcal{F}^T \)}
\caption{$\fmulti$}
\end{algorithm}

\begin{proof}[Proof of \cref{lem:fair-power-of-two}]
    To prove this lemma, we need to define the following term
    \begin{itemize}
        \item $i$th fairness constraint of algorithm $\fptwo$: A clustering $\m{C}$ is said to satisfy the $i$th fairness constraint of algorithm $\fptwo$ if, for every cluster $C \in \m{C}$ and for all pairs of distinct colors $c_k \ne c_\ell$ in the set $B^i_j$, the number of points of color $c_k$ in $C$ equals the number of points of color $c_\ell$ in $C$; that is,
        \[
            c_k(C) = c_\ell(C).
        \]
    \end{itemize}

    Now we prove the following claim
    \begin{claim}\label{clm:main}
        In the algorithm $\fptwo$, the clustering $\m{N}^i$ is $2$-close clustering to $\m{N}^{i-1}$ that satisfies the $i$th fairness constraint of $\fptwo$, where $\m{N}^0 = \m{D}$.
    \end{claim}

    We will prove \cref{clm:main} later. First, let us prove \cref{lem:fair-power-of-two} assuming \cref{clm:main}.

     To prove \cref{lem:fair-power-of-two}, we need to prove $\m{N}^{\log |\chi|}$, the output of the algorithm $\fptwo$ is $O(\card{\chi}^{1.6})$-close to $\m{D}$. Let, $\m{N}^*$ be the closest fair clustering to $\m{D}$.

    We will prove using mathematical induction that after any iteration $i$ of the algorithm $\fptwo$ we have 

    \begin{align}
        \dist(\m{D}, \m{N}^i) \leq (3^i - 1) \dist(\m{D}, \m{N}^*) \label{equn:analysis-fptwo-one}
    \end{align}

    For the base case, that is when $i = 1$ by \cref{clm:main} we have $\m{N}^1$ is a $2$-close clustering to $\m{D}$ that satisfies the $1$st fairness constraint of algorithm $\fptwo$. Since $\m{N}^*$ also satisfies the $1$st fairness constraint we have
    \[
        \dist(\m{D}, \m{N}^1) \leq 2 \dist(\m{D}, \m{N}^*) 
    \]

    Now, let us assume \cref{equn:analysis-fptwo-one} is true for $i = k-1$, that is
    \[
        \dist(\m{D}, \m{N}^{k - 1}) \leq (3^{k - 1} - 1) \dist(\m{D}, \m{N}^*)
    \]

    Now, we prove it for $i = k$.

    \begin{align}
        \dist(\m{N}^{k - 1}, \m{N}^k) &\leq 2 \dist(\m{N}^{k - 1}, \m{N}^*) \label{equn:analysis-fptwo-two} \\
        &\leq 2(\dist(\m{N}^{k - 1}, \m{D}) + \dist(\m{D}, \m{N}^*)) \n \\ &\text{(triangle inequality)} \nonumber \\
        &\leq 2((3^{k - 1} - 1)\dist(\m{D}, \m{N}^*) + \dist(\m{D}, \m{N}^*)) \n \\ &\text{(by IH)} \nonumber \\
        &= 2 \cdot 3^{k - 1} \dist(\m{D}, \m{N}^*) \label{equn:analysis-fptwo-two-(i)}
    \end{align}

    Here, \cref{equn:analysis-fptwo-two} is true because $\m{N}^*$ also satisfies the $i$th fairness constraint of the algorithm $\fptwo$ and due to \cref{clm:main}.

    Now, we have
    \begin{align}
        \dist(\m{D}, \m{N}^k) &\leq \dist(\m{D}, \m{N}^{k - 1}) + \dist(\m{N}^{k - 1}, \m{N}^k) \n \\
        &\text{(triangle inequality)} \nonumber \\
        &\leq (3^{k - 1} - 1) \dist(\m{D}, \m{N}^*) + 2 \cdot 3^{k - 1} \dist(\m{D}, \m{N}^*) \nonumber \\
        &\text{(by IH and \cref{equn:analysis-fptwo-two-(i)})} \nonumber \\
        &= (3^k - 1) \dist(\m{D}, \m{N}^*) \label{equn:analysis-fptwo-three}
    \end{align}

    Hence, now we can conclude for $i = \log |\chi|$, 
    \begin{align}
        \dist(\m{D}, \m{N}^{\log |\chi|}) &\leq (3^{\log |\chi|} - 1) \dist(\m{D}, \m{N}^*) \n \\
        &= O(|\chi|^{1.6}) \dist(\m{D}, \m{N}^*) \, \, (\log_23 = 1.6)
    \end{align}

    Thus the algorithm $\fptwo$ computes a $O(\chi^{1.6})$-close $\fair$ which completes the proof of \cref{lem:fair-power-of-two}.

    Now, let us prove \cref{clm:main}. To prove \cref{clm:main}, we establish two intermediate claims: \cref{clm:lower-bound-opt} and \cref{clm:upper-bound-algo}.

    Let $\m{N}^{i^*}$ denote the closest clustering to $\m{N}^{i - 1}$ that satisfies the $i$th fairness constraint of algorithm $\fptwo$.

    In \cref{clm:lower-bound-opt}, we derive a lower bound on the distance $\dist(\m{N}^{i - 1}, \m{N}^{i^*})$, and in \cref{clm:upper-bound-algo}, we provide an upper bound on the distance $\dist(\m{N}^{i - 1}, \m{N}^i)$, where $\m{N}^i$ is the clustering obtained after the $i$th iteration of algorithm $\fptwo$, starting from the initial clustering $\m{D}$.

    By comparing the bounds from \cref{clm:lower-bound-opt} and \cref{clm:upper-bound-algo}, we will complete the proof of \cref{clm:main}.

    To formalize \cref{clm:lower-bound-opt}, lets recall and introduce some notations:

    \begin{itemize}
        \item Recall, $B^{i - 1}_k(N_a^{i - 1})$ denotes the set of data points in the cluster $N_a^{i - 1} \in \m{N}^{i - 1}$ whose colors belong to the block $B^{i - 1}_k$.
    
        \item Recall, $s(B_{2k - 1}^{i-1}(N_a^{i-1}), B_{2k}^{i - 1}(N_a^{i - 1})) = \Sn (say)$ denotes the surplus between $B_{2k - 1}^{i-1}$ and $B_{2k}^{i - 1}$ in cluster $N_a^{i-1}$. Specifically, 
        \begin{itemize}
            \item If $|B_{2k - 1}^{i-1}(N_a^{i-1})| \geq |B_{2k}^{i}(N_a^{i-1})|$, then
            \[
                \Sn \subseteq B_{2k - 1}^{i-1}(N_a^{i-1}), \quad |\Sn| = |B_{2k - 1}^{i-1}(N_a^{i-1})| - |B_{2k}^{i - 1}(N_a^{i-1})|,
            \]
            such that $\forall c_j \neq c_\ell \in B_{2k}^{i-1}$, we have $c_j(\Sn) = c_\ell(\Sn)$.
            \item Otherwise,
            \[
                \Sn \subseteq B_{j+1}^{i - 1}(N_a^{i-1}), \quad |\Sn| = |B_{j+1}^{i - 1}(N_a^{i-1})| - |B_j^{i-1}(N_a^{i-1})|,
            \]
            such that $\forall c_k \neq c_\ell \in B_{j+1}^{i - 1}$, we have $c_k(\Sn) = c_\ell(\Sn)$.
        \end{itemize}
        \item Let, $T_a$ denotes the union of the surpluses within the cluster $N_a^{i - 1}$, computed across all paired color blocks $(B_{2k - 1}^{i - 1}, B_{2k}^{i - 1})$ at iteration $i - 1$. More specifically,
        \begin{align*}
            T_a &= \bigcup_{k = 1}^{|\chi| / 2^{i}} \Sn
        \end{align*}
    \end{itemize}

\begin{claim}\label{clm:lower-bound-opt}
\begin{align}
\dist(\m{N}^{i - 1}, \m{N}^{i^*}) \geq 
\frac{1}{2}\sum_{N_a^{i - 1} \in \m{N}^{i - 1}}  
&|T_a| \cdot \left(\left|N_a^{i - 1}\right| - \left| T_a \right| \right) +  |T_a|^2 \nonumber
\end{align}
\end{claim}

\begin{proof}
    Consider a cluster $N_a^{i - 1} \in \m{N}^{i - 1}$. Suppose in $\m{N}^{i^*}$ the cluster $N_a^{i - 1}$ is partitioned into $X_1, X_2, \ldots, X_t$, more specifically, 

    \begin{itemize}
        \item For all $j \in [t]$, $X_j \subseteq N_{r_j}^{i^*}$ for some $N_{r_j}^{i^*} \in \m{N}^{i^*}$.
        \item For all $j \neq \ell \in [t]$, we have $N_{r_j}^{i^*} \neq N_{r_\ell}^{i^*}$.
        \item $\bigcup_{j \in [t]} X_j = N_a^{i - 1}$.
    \end{itemize}

    By abuse of notation, let us define $B^{i - 1}_{2k - 1}(X_j)$ and $B^{i - 1}_{2k}(X_j)$ be the set of points in $X_j$ that has a color from the blocks $B^{i - 1}_{2k - 1}$ and $B^{i - 1}_{2k}$ respectively. 

    WLOG assume, $|B^{i - 1}_{2k - 1}(N_a^{i - 1})| > |B^{i - 1}_{2k}(N_a^{i - 1})|$.

    Recall, the blocks are created in such a way that for two consecutive blocks $B^{i - 1}_{2k - 1}$ and $B^{i - 1}_{2k}$, we have 
    \[
        |B^{i - 1}_{2k - 1}| = |B^{i - 1}_{2k}|.
    \]

    Let's create arbitrary pairing of colors $(c,\hat{c})$ where $c \in B^{i - 1}_{2k - 1}$ and $\hat{c} \in B^{i - 1}_{2k}$, we call $\hat{c} \in B^{i - 1}_{2k}$ a color corresponding to $c \in B^{i - 1}_{2k - 1}$.

    \begin{itemize}
        \item \textbf{Surplus between two colors $c \in B^{i - 1}_{2k - 1}$ and $\hat{c} \in B^{i - 1}_{2k - 1}$ in a cluster $N_a^{i - 1}$} is defined as
        \[
            s^a(c,\hat{c}) \subseteq B^{i - 1}_{2k - 1}(N_a^{i - 1})
        \]
        of size $\max(0, c(N_a^{i - 1}) - \hat{c}(N_a^{i - 1}))$
        \item \textbf{Surplus between two colors $c \in B^{i - 1}_{2k - 1}$ and $\hat{c} \in B^{i - 1}_{2k}$ for a partition $X_j$} is defined as
        \[
            \sigma^j_c \subseteq B^{i - 1}_{2k - 1}(X_j)
        \]
        of size $\max(0, c(X_j) - \hat{c}(X_j))$
        It is straightforward to see that 
        \[
            \sum_{j = 1}^t |\sigma^j_c| \geq |s^a(c,\hat{c})|
        \]

        Let, us assume $y \in [t]$ be an index such that
    \[
        \sum_{j = 1}^{y - 1}|\sigma^j_c| < |s^a(c,\hat{c})| \leq \sum_{j = 1}^y |\sigma^j_c|
    \]

    Again assume, 
    \[
    \hat{\sigma^y_c} \subseteq \sigma^y_c
    \]

    such that,

    \[
        \sum_{j = 1}^{y - 1}|\sigma^j_c| + |\hat{\sigma^y_c}| = |s^a(c,\hat{c})|
    \]

    Let us now redefine the notation $\sigma^j_c$ for $j \in [t]$

    \[
    \sigma^j_c := 
        \begin{cases}
        \sigma^j_c & \text{if } j < y, \\
        \hat{\sigma^y_c} & \text{if } j = y, \\
        \emptyset & \text{if } j > y.
        \end{cases}
    \]

    That is, the sets $\sigma^j_c$ from $j = 1$ to $(y - 1)$ remains unchanged. We shrink the set $\sigma^y_c$ to include only as many elements as needed to make the summation $|s^a(c,\hat{c})|$. We ignore the sets $\sigma^j_c$ from $j = (y + 1)$ to $t$ completely. 

        \item \textbf{Surplus with respect to consecutive pair of blocks $B^{i - 1}_{2k - 1}$ and $B^{i - 1}_{2k}$ for a partition $X_j$} is defined as
        \[
            \Sx = \bigcup_{c \in B^{i - 1}_{2k - 1}} \sigma^j_c
        \]
        \item \textbf{Surplus with respect to a partition $X_j$} is defined as
        \[
            S_j = \bigcup_{k = 1}^{m_{i - 1}} S^k_j
        \]
        where $m_{i - 1}$ is the number of blocks created at iteration $(i - 1)$.
    \end{itemize}
    It is straightforward to see that
    \[
        \sum_{j = 1}^t |S_j| = T_a
    \]

    Now, since $\m{N}^{i^*}$ satisfies $i$th fairness constraint of $\fptwo$, hence in $N_{r_j}^{i^*} \in \m{N}^{i^*}$ we have,
    
    \[
        B_{2k - 1}^{i - 1}(N_{r_j}^{i^*}) = B_{2k}^{i - 1}(N_{r_j}^{i^*}). 
    \]

    Recall, $X_j \subseteq N_{r_j}^{i^*}$, hence there must exist at least $|\Sx|$ vertices having colors from the color block $B^{i - 1}_{2k}$ in $N_{r_j}^{i^*}$ that belongs to clusters other than $N_a^{i - 1}$. Let us denote this set of vertices by $M^k_j$. More specifically,

    \begin{itemize}
        \item $M^k_j \subseteq N_{r_j}^{i^*}$ such that following conditions are satisfied.
        \begin{enumerate}
            \item $M^k_j \cap N_a^{i - 1} = \emptyset$.
            \item $|M^k_j| = |\Sx|$.
            \item The vertices in $M^k_j$ have colors from the color block $B^{i - 1}_{2k}$.
        \end{enumerate}
    \end{itemize}

    Let,
    \[
        M_j = \bigcup_{k = 1}^{m_{i - 1}} M^k_j
    \]
    Note, since $|M^k_j| = |\Sx|$ we also have $|M_j| = |S_j|$.

    Let us also define
    \[
        M(N_a^{i - 1}) = \bigcup_{j = 1}^t M_j
    \]

    By the definition of $\dist(\m{N}^{i - 1}, \m{N}^{i^*})$ is the number of pairs $(u,v)$ that are separated in $\m{N}^{i - 1}$ and together in $\m{N}^{i^*}$ or viceversa. Formally, we say that $(u,v)$ are together in $\m{N}^{i^*}$ if there exists $N_k^{i^*} \in \m{N}^{i^*}$ such that $u,v \in N_k^{i^*}$ and we say $(u,v)$ are separated in $\m{N}^{i^*}$ if there exists $N_k^{i^*}, N_j^{i^*} \in \m{N}^{i^*}$ such that $k \neq j$, $u \in N_k^{i^*}$ and $v \in N_j^{i^*}$

    Now, to lower bound $\dist(\m{N}^{i - 1}, \m{N}^{i^*})$ let us define some costs.
    \begin{itemize}
        \item $\costone(N_a^{i - 1}):$ For a cluster $N_a^{i - 1} \in \m{N}^{i - 1}$, $\costone(N_a^{i - 1})$ denotes the number of pairs $(u,v)$ such that 
        \begin{enumerate}[(a)]
            \item $u \in (N_a^{i - 1} \setminus T_a)$ and $v \in T_a$ but separated in $\m{N}^{i^*}$  or,
            \item $u \in (N_a^{i - 1} \setminus T_a$ and $v \in M(N^{i - 1}_a)$.
        \end{enumerate}
        
        \item $\costtwo(N_a^{i - 1}):$ For a cluster $N_a^{i - 1} \in \m{N}^{i - 1}$, $\costtwo(N_a^{i - 1})$ denotes the number of pairs $(u,v)$ such that
        \begin{enumerate}[(a)]
            \item $u, v \in T_a$ but separated in $\m{N}^{i^*}$  or,
            \item $u \in T_a$ and $v \in M(N_a^{i - 1})$.
        \end{enumerate}
    \end{itemize}

    We can verify that the pairs counted in $\costone(N_a^{i - 1}, k)$, and $\costtwo(N_a^{i - 1}, k)$ are disjoint and thus we have.

    \begin{align}
        \dist(\m{N}^{i - 1}, \m{N}^{i^*}) \geq &\sum_{N_a^{i - 1} \in \m{N}^{i - 1}} \frac{1}{2} \costone(N_a^{i - 1}) \n \\ &+ \frac{1}{2}\costtwo(N_a^{i - 1}) \label{equ:clm-lower-bound-opt}
    \end{align}

    We multiply $\costone(N_a^{i - 1})$ and $\costtwo(N_a^{i - 1})$ with $1/2$ to avoid overcounting of pairs. In both the costs we count the pairs $(u,v)$ in $N_a^{i - 1}$ and $M(N_a^{i - 1})$. This pair $(u,v)$ may be counted twice because we take summation overall $N_a^{i - 1} \in \m{N}^{i - 1}$.

    Now, to prove \cref{clm:lower-bound-opt}, we prove the following

    \begin{enumerate}
        \item $\costone(N_a^{i - 1}) \geq |T_a| \left(\left|N_a^{i - 1}\right| - \left| T_a \right| \right)$.
        \item $\costtwo(N_a^{i - 1}) \geq \left|T_a \right|^2$.
    \end{enumerate}

    It is easy to see that combining \cref{equ:clm-lower-bound-opt} and the above statements will prove this \cref{clm:lower-bound-opt}. So, let us now prove the above statements.

    \begin{claim}
        $\costone(N_a^{i - 1}) \geq |T_a| \left(\left|N_a^{i - 1}\right| - |T_a| \right)$
    \end{claim}

    \begin{proof}
        Recall in $N^{i^*}$, the cluster $N_a^{i - 1}$ is partitioned into $X_1, \ldots, X_t$.

        Now, consider a partition $X_j$ of $N_a^{i-1}$ where $j \in [t]$. Let us count the number of pairs $(u,v)$ such that $u \in X_j \setminus S_j$ and $v \in S_\ell$ for $\ell \neq j$. The number of such pairs is

        \begin{align}
	   \left|X_j \setminus S_j\right| \sum_{\ell \neq j}|S_\ell| \label{eq:cost-one-eq-one}
        \end{align}

        Let us also count the number of pairs $(u,v)$ such that $ u \in X_j \setminus S_j$ and $ v \in M_j$. The number of such pairs is 

        \begin{align}
            &|X_j \setminus S_j| |M_j| \n \\
            = &|X_j \setminus S_j| |S_j| \label{eq:cost-one-eq-two} 
        \end{align}

        Now, combining \cref{eq:cost-one-eq-one} and \cref{eq:cost-one-eq-two} we get the number of pairs $(u,v)$ such that,

        \[
            (u,v) \in X_j \setminus S_j \times S_\ell \, \, \text{for $\ell \neq j$}
        \]

        or

        \[
            (u,v) \in X_j \setminus S_j \times M_j
        \]

        is

        \begin{align}
            &\left|X_j \setminus S_j\right| \sum_{\ell \neq j}|S_\ell| \n \\
            + &|X_j \setminus S_j| |S_\ell| \n \\
            = &\left|X_j \setminus S_j\right| |T_a| \label{eq:cost-one-eq-three}
        \end{align}

        Now from \cref{eq:cost-one-eq-three} we get,

        \begin{align}
            \costone(N_a^{i - 1}, k) &\geq \sum_{j = 1}^t \left|X_j \setminus S_j\right| |T_a| \n \\
            &= |T_a| \sum_{j = 1}^t \left|X_j \setminus S_j\right| \n \\
            &= \left|T_a\right| \left(\left|N_a^{i - 1}\right| - |T_a|\right) \n
        \end{align}
 
    \end{proof}
    
    \begin{claim}
        $\costtwo(N_a^{i - 1}) \geq \left|T_a \right|^2$
    \end{claim}

    \begin{proof}
        Consider a partition $X_j$ of $N_a^{i-1}$ where $j \in [t]$. Let us count the number of pairs $(u,v)$ such that $u \in S_j$ and $v \in S_\ell$ for $\ell \neq j$. The number of such pairs is

        \begin{align}
	   \left|S_j\right| \sum_{\ell \neq j}|S_\ell| \label{eq:cost-two-eq-one}
        \end{align}

        Let us also count the number of pairs $(u,v)$ such that $ u \in S_j$ and $ v \in M_j$. The number of such pairs is 

        \begin{align}
            &|S_j| |M_j| \n \\
            = &|S_j| |S_j| \label{eq:cost-two-eq-two} 
        \end{align}

        Now, combining \cref{eq:cost-two-eq-one} and \cref{eq:cost-two-eq-two} we get the number of pairs $(u,v)$ such that,

        \[
            (u,v) \in S_j \times S_\ell \, \, \text{for $\ell \neq j$}
        \]

        or

        \[
            (u,v) \in S_j \times M_j
        \]

        is

        \begin{align}
            &\left|S_j\right| \sum_{\ell \neq j}|S_\ell|
            + |S_j| |S_j| \n \\
            = &\left|S_j\right| |T_a| \label{eq:cost-two-eq-three}
        \end{align}

        Now from \cref{eq:cost-two-eq-three} we get,

        \begin{align}
            \costtwo(N_a^{i - 1}) &\geq \sum_{j = 1}^t \left|S_j\right| |T_a| \n \\
            &= |T_a| \sum_{j = 1}^t \left|S_j\right| \n \\
            &= \left|T_a\right| \left|T_a\right|\n \\
            &= \left|T_a\right|^2 \n
        \end{align}
    \end{proof}
    
\end{proof}

\begin{claim}\label{clm:upper-bound-algo}
    \begin{align}
        \dist(\m{N}^{i - 1}, \m{N}^i) \leq &\sum_{N_a^{i - 1} \in \m{N}^{i - 1}} \left|T_a\right| \left(\left|N_a^{i - 1}\right| - \left| T_a \right| \right) \n \\
        &+ \frac{1}{2} \left|T_a \right|^2 \n
    \end{align}
\end{claim}

\begin{proof}
    In the algorithm $\fptwo$, from each cluster $N_a^{i - 1} \in \m{N}^{i - 1}$ we cut the set $T_a$.

    Hence, the pairs $(u, v)$ s.t. $u \in N_a^{i - 1} \setminus T_a$ and $v \in T_a$ are counted in $\dist(\m{N}^{i - 1}, \m{N}^i)$. 

    The number of such pairs is
    \begin{align}
        \left|T_a\right| \left(\left|N_a^{i - 1}\right| - \left| T_a \right| \right) \label{eq:upper-bound-equn-one}
    \end{align}

    Again in~\cref{algo:greedymerge} the set $T_a$ can further get splitted into multiple subsets $R_1, R_2, \ldots, R_t$ (say). Each of these sets $R_i$ for $i \in [t]$ gets merged with $|R_i|$ points from a different cluster.

    Hence, the following pairs $(u,v)$ are counted in $\dist(\m{N}^{i - 1}, \m{N}^i)$ which satisfies 

    \begin{enumerate}
        \item $u \in R_i$ and $v$ belongs to the set that merged to $R_i$.
        \item $u \in R_i$ and $v \in R_j$ for $i \neq j$.
    \end{enumerate}

    To count such pairs $(u,v)$ we use a charging scheme, we charge $1/2$ for the vertex $u$ and $1/2$ for the vertex $v$. That is we define for a set $R_i$. 

    \begin{align}
        \pay{R_i} = \frac{1}{2} |\{ (u,v) \mid &u \in R_i \, \, \text{and $v \in$ the set merged to $R_i$} \n \\
        &\text{ or $u \in R_i$ and $v \in R_j$ for $i \neq j$}\} |\n
    \end{align}

    The total number of such pairs is
    \begin{align}
        \sum_{i = 1}^t \pay{R_i} &= \sum_{i = 1}^t \frac{1}{2} |R_i|^2 + \frac{1}{2} |R_i||(T_a \setminus R_i)| \n \\
        &= \frac{1}{2} |T_a| \sum_{i = 1}^t |R_i| \n \\
        &= \frac{1}{2} |T_a|^2 \label{eq:upper-bound-equn-two}
    \end{align}

    By \cref{eq:upper-bound-equn-one} and \cref{eq:upper-bound-equn-two} we get

    \begin{align}
        \dist(\m{N}^{i - 1}, \m{N}^i) \leq &\sum_{N_a^{i - 1} \in \m{N}^{i - 1}} \left|T_a\right| \left(\left|N_a^{i - 1}\right| - \left| T_a \right| \right) \n \\
        &+ \frac{1}{2} \left|T_a \right|^2 \n
    \end{align}
\end{proof}

It is straightforward to see that \cref{clm:lower-bound-opt} and \cref{clm:upper-bound-algo} proves \cref{clm:main}.

\end{proof}

\begin{proof}[Proof of \cref{lem:analyze-fmulti}]
    To prove this lemma, let us define the $t$th fairness constraint of $\fmulti$ algorithm.

    \begin{itemize}
        \item $t$th fairness constraint of $\fmulti$ : Let \( \mathcal{C} \) be a clustering and let \( \{ B_k^t \} \) denote the color blocks at the $t$th iteration in the $\fmulti$ algorithm. We say that \( \mathcal{C} \) satisfies the \( t \)th Fairness Constraint of $\fmulti$ if, for every cluster \( C_i \in \mathcal{C} \) and every block \( B_k^t \), the color counts satisfy:

\[
z_{a_1}(C_i) : z_{a_2}(C_i) : \cdots : z_{a_m}(C_i) = p_{a_1} : p_{a_2} : \cdots : p_{a_m}
\]
where $B_k^t = \{ z_{a_1}, \ldots, z_{a_m} \}$

    \end{itemize}

To prove \cref{lem:analyze-fmulti}, we need to prove the following claim

\begin{claim}\label{clm:each-step-4-approx}
    For all iterations $t$ of the algorithm $\fmulti$, $\m{F}^t$ is a $6$-close clustering to $\m{F}^{t - 1}$ that satisfies the $t$th fairness constraint of $\fmulti$. More specifically,
    \[
        \dist(\m{F}^{t - 1}, \m{F}^t) \leq 6 \dist(\m{F}^{t - 1}, \m{F}^{t^*}) 
    \]
    where $\m{F}^{t^*}$ is the closest clustering to $\m{F}^{t - 1}$ that satisfies $t$th fairness constraint of $\fmulti$.
\end{claim}

For now, let us assume \cref{clm:each-step-4-approx} and prove \cref{lem:analyze-fmulti}.

\paragraph{Proof of \cref{lem:analyze-fmulti}:} To prove \cref{lem:analyze-fmulti}, we need to prove $\m{F}^{\log r}$, the output of the algorithm $\fmulti$ is $r^{2.8}$-close to $\m{I}$. Let $\m{F}^*$ be the closest fair clustering to $\m{I}$.

    We will prove using mathematical induction that after any iteration $t$ of the algorithm $\fmulti$ we have 

    \begin{align}
        \dist(\m{I}, \m{F}^t) \leq (7^t - 1) \dist(\m{I}, \m{F}^*) \label{equn:analysis-fptwo-one-multi}
    \end{align}

    For the base case, that is when $t = 1$ by \cref{clm:each-step-4-approx} we have $\m{F}^1$ is a $6$-close clustering to $\m{I}$ that satisfies the $1$st fairness constraint of algorithm $\fmulti$. Since $\m{F}^*$ also satisfies the $1$st fairness constraint we have
    \[
        \dist(\m{D}, \m{F}^1) \leq 6\dist(\m{D}, \m{F}^*) 
    \]

    Now, let us assume \cref{equn:analysis-fptwo-one-multi} is true for $t = k-1$, that is
    \[
        \dist(\m{I}, \m{F}^{k - 1}) \leq (7^{k - 1} - 1) \dist(\m{I}, \m{F}^*)
    \]

    Now, we prove it for $t = k$.

    \begin{align}
        \dist(\m{F}^{k - 1}, \m{F}^k) &\leq 6 \dist(\m{F}^{k - 1}, \m{F}^*)\label{equn:analysis-fmulti-two} \\
        &\leq 6 (\dist(\m{F}^{k - 1}, \m{I}) + \dist(\m{I}, \m{F}^*)) \n \\ &(\text{by triangle inequality}) \n \\
        &\leq 6 \cdot (7^{k - 1} - 1) \dist(I,\m{F}^*) + 6 \dist(I,\m{F}^*) \n \\
        &(\text{by inductive hypothesis}) \n \\
        &= 6 \cdot 7^{k - 1} \dist(\m{I},\m{F}^*) \label{equn:analysis-fmulti-two-(i)}
    \end{align}

    Here, \cref{equn:analysis-fmulti-two} is true because $\m{F}^*$ also satisfies the $k$th fairness constraint of the algorithm $\fmulti$ and due to \cref{clm:each-step-4-approx}.

    Now, we have
    \begin{align}
        \dist(\m{I}, \m{F}^k) &\leq \dist(\m{I}, \m{F}^{k - 1}) + \dist(\m{F}^{k - 1}, \m{F}^k) \n \\
        &(\text{by triangle inequality})\n \\
        &\leq (7^{k - 1} - 1) \dist(\m{I}, \m{F}^*) + 6 \cdot 7^{k - 1} \dist(\m{I}, \m{F}^*)  \n \\
        & (\text{by inductive hypothesis and \cref{equn:analysis-fmulti-two-(i)}}) \n \\
        &= (7^k - 1) \dist(\m{I}, \m{F}^*) \label{equn:analysis-fmulti-three}
    \end{align}

    Hence, now we can conclude for $t = \log r$, 
    \begin{align}
        \dist(\m{I}, \m{F}^{\log r}) &\leq O(7^{\log r}) \dist(\m{I}, \m{F}^*) \n \\
        &= O(r^{2.8}) \dist(\m{I}, \m{F}^*) \, \, (\textbf{as} \log_27 = 2.8)
    \end{align}

    Thus the algorithm $\fmulti$ computes a $O(r^{2.8})$-close $\fair$ which completes the proof of \cref{lem:analyze-fmulti}.

    Now, let us prove \cref{clm:each-step-4-approx}. To prove \cref{clm:each-step-4-approx}, we establish four intermediate claims: \cref{clm:lower-bound-opt-multi-one}, \cref{clm:lower-bound-opt-multi-two}, \cref{clm:lower-bound-opt-new} and \cref{clm:upper-bound-algo-multi}.

    Let $\m{F}^{t^*}$ denote the closest clustering to $\m{F}^{t - 1}$ that satisfies the $t$th fairness constraint of algorithm $\fmulti$.

    In \cref{clm:lower-bound-opt-multi-one}, \cref{clm:lower-bound-opt-multi-two} and \cref{clm:lower-bound-opt-new} we derive lower bounds on the distance $\dist(\m{F}^{t - 1}, \m{F}^{t^*})$, and in \cref{clm:upper-bound-algo-multi}, we provide an upper bound on the distance $\dist(\m{F}^{t - 1}, \m{F}^t)$, where $\m{F}^t$ is the clustering obtained after the $t$th iteration of algorithm $\fmulti$, starting from the initial clustering $\m{I}$.

    By comparing the bounds from \cref{clm:lower-bound-opt-multi-one}, \cref{clm:lower-bound-opt-multi-two}, \cref{clm:lower-bound-opt-new} and \cref{clm:upper-bound-algo-multi}, we will complete the proof of \cref{clm:each-step-4-approx}.

    To state \cref{clm:lower-bound-opt-multi-one}, we define the surplus and deficit for a cluster \( F_a^{t - 1} \in \mathcal{F}^{t - 1} \).

    Recall \( \mathcal{F}^{t - 1} \) denote the clustering obtained after the \((t - 1)\)th iteration of the $\fmulti$ algorithm and \( \{ B^{t - 1}_1, B^{t - 1}_2, \ldots, B^{t - 1}_{m_{t - 1}} \} \) be the set of color blocks at this iteration. For a color \( c_u \in \chi \), recall, \( p_u \) denote its proportion in the vertex set \( V \), i.e.,
\[
c_1(V): c_2(V): \cdots : c_{|\chi|}(V) = p_1 : p_2 : \cdots : p_{|\chi|}.
\]

WLOG, we assume $p_1 > p_2 >\cdots > p_{|\chi|}$. Note that due to this assumption for two colors $c_u \in B^{t - 1}_{j}$ and $c_v \in B^{t - 1}_{k}$ where $k > j$ we have $p_v < p_u$.

We define the following for a cluster \( F_a^{t - 1} \in \mathcal{F}^{t - 1} \):

Recall the color blocks $B_{2k -1}^{t - 1}$ and $B_{2k}^{t - 1}$ are created in such a way that we have $|B^{t - 1}_{2k - 1}| \geq |B^{t - 1}_{2k}|$. Suppose $\tilde{B^{t - 1}_{2k - 1}} \subseteq B^{t - 1}_{2k - 1}$ such that $|\tilde{B^{t - 1}_{2k - 1}}| = |B^{t - 1}_{2k}|$.

We create an arbitrary pair of colors $(c_j, c_{\hat{j}})$ where $c_j \in \tilde{B^{t - 1}_{2k - 1}}$ and $c_{\hat{j}} \in B^{t - 1}_{2k}$. We say $c_{\hat{j}} \in B^{t - 1}_{2k}$ is a color corresponding to $c_j \in B^{t - 1}_{2k - 1}$. Note for the colors present in $B^{t - 1}_{2k - 1} \setminus \tilde{B^{t - 1}_{2k - 1}}$ we have no corresponding color.

Now we define the following terms

\begin{itemize}
    \item \textbf{Weight of a vertex $v$:} For $v \in V$, suppose it has a color $c_j \in B^{t - 1}_{2k - 1}$ and its corresponding color $c_{\hat{j}} \in B^{t - 1}_{2k}$.

    In this case, since $p_j > p_{\hat{j}}$ we define 
    \[
        w(v) = \frac{p_{\hat{j}}}{p_j}
    \]
    otherwise if $v$ has color $c_{\hat{j}} \in B_{2k}^{t - 1}$ and its corresponding color $c_j \in B_{2k - 1}^{t - 1}$ then we define
    \[
        w(v) = 1
    \]
    Note, $w(v) \leq 1$ for all $v \in V$.
    \item \textbf{Surplus w.r.t. two corresponding colors $c_j$ and $c_{\hat{j}}$ in a cluster $F_a^{t - 1} \in \m{F}^{t - 1}$} is defined as
    \[
        s^a(c_j, c_{\hat{j}}) = \max( 0, c_{\hat{j}}(F_a^{t - 1}) - \frac{p_{\hat{j}}}{p_j}c_j(F_a^{t - 1}))
    \]
    \item \textbf{Surplus w.r.t. consecutive blocks $B^{i - 1}_{2k - 1}$ and $B^{i - 1}_{2k}$ in a cluster $F_a^{t - 1} \in \m{F}^{t - 1}$} is defined as
    \[
        T^k_a = \sum_{c_j \in B^{i - 1}_{2k - 1}} s^a(c_j, c_{\hat{j}})
    \]
    \item \textbf{Surplus of a cluster $F_a^{t - 1} \in \m{F}^{t - 1}$} is defined as
    \[
        T_a = \sum_{k = 1}^{m_{t - 1}} T^k_a
    \]
    where $m_{t - 1}$ is the number of blocks at iteration $(t - 1)$.
    \item \textbf{deficit w.r.t. two corresponding colors $c_j$ and $c_{\hat{j}}$ in a cluster $F_a^{t - 1} \in \m{F}^{t - 1}$} is defined as
    \[
        d^a(c_j, c_{\hat{j}}) = \max (0, \frac{p_{\hat{j}}}{p_j}c_j(F_a^{t - 1}) - c_{\hat{j}}(F_a^{t - 1}))
    \]
    \item \textbf{deficit w.r.t. consecutive blocks $B^{i - 1}_{2k - 1}$ and $B^{i - 1}_{2k}$ in a cluster $F_a^{t - 1} \in \m{F}^{t - 1}$} is defined as
    \[
        D^k_a = \sum_{c_j \in B^{i - 1}_{2k - 1}} d^a(c_j, c_{\hat{j}})
    \] 
    \item \textbf{deficit of a cluster $F_a^{t - 1} \in \m{F}^{t - 1}$} is defined as
    \[
        D_a = \sum_{k = 1}^{m_{t - 1}} D^k_a
    \]
\end{itemize}

    Now, we can state \cref{clm:lower-bound-opt-multi-one}.

    \begin{claim}\label{clm:lower-bound-opt-multi-one}
        \begin{align}
        \dist(\m{F}^{t - 1}, \m{F}^{t^*}) \geq \frac{1}{4} &\sum_{F_a^{t - 1} \in \m{F}^{t - 1}}  T_a \left(\left|F_a^{t - 1}\right| - T_a \right) \n \\
        &+ D_a \left(\left|F_a^{t - 1}\right| -  T_a \right) \n
    \end{align}
    \end{claim}

    \begin{proof}
        Consider a cluster $F_a^{t - 1} \in \m{F}^{t - 1}$. Suppose in $\m{F}^{t^*}$ the cluster $F_a^{t - 1}$ is partitioned into $X_1, X_2, \ldots, X_s$, more specifically, 

    \begin{itemize}
        \item For all $j \in [t]$, $X_j \subseteq F_{r_j}^{t^*}$ for some $F_{r_j}^{t^*} \in \m{F}^{t^*}$.
        \item For all $j \neq \ell \in [t]$, we have $F_{r_j}^{t^*} \neq F_{r_\ell}^{t^*}$.
        \item $\bigcup_{j \in [s]} X_j = F_a^{t - 1}$.
    \end{itemize}

    According to the definition of $\dist(\m{F}^{t - 1}, \m{F}^{t^*})$, in $\dist(\m{F}^{t - 1}, \m{F}^{t^*})$ we count the number of pairs $(u,v)$ that are present in different clusters in the clustering $\m{F}^{t - 1}$ but in the same cluster in the clustering $\m{F}^{t^*}$ or present in the same cluster in the clustering $\m{F}^{t - 1}$ but present in different clusters in the clustering $\m{F}^{t^*}$.

    To give a lower bound on $\dist(\m{F}^{t - 1}, \m{F}^{t^*})$ we calculate the cost incurred by each cluster $F_a^{t - 1} \in \m{F}^{t - 1}$. Let us describe this formally,
    \begin{align}
        \pay{F_a^{t - 1}} = |&\{(u,v) \mid (u,v \in F_a^{t - 1} \land u \nsim_{\m{F}^{t^*}} v) \n \\
        & \lor (u \in F_a^{t - 1}, v \notin F_a^{t - 1} \land u \sim_{\m{F}^{t^*}} v) \}| \n
    \end{align}
    where $u \sim_{\m{F}^{t^*}} v$ denotes $u$ and $v$ belongs to the same cluster in the clustering $\m{F}^{t^*}$.

    It is straightforward to see that 
    \begin{align}
        \dist(\m{F}^{t - 1}, \m{F}^{t^*}) \geq \frac{1}{2} \sum_{F_a^{t - 1} \in \m{F}^{t - 1} }\pay{F_a^{t - 1}} \label{eq:main-multi}
    \end{align}

    In the above expression, we multiply by $1/2$ because we took summation overall $F_a^{t - 1} \in \m{F}^{t - 1} $. To count a pair $(u,v)$ where $u \in F_a^{t - 1}$ and $v \in F_b^{t - 1}$ for $a \neq b$ we charge $1/2$ when calculating $\pay{F_a^{t - 1}}$ and again $1/2$ when calculating $\pay{F_b^{t - 1}}$. 

    Now, we provide a lower bound on $\pay{F_a^{t - 1}}$. To provide this lower bound, we need to define some terms.

    \begin{itemize}
    \item \textbf{Weight of a color block:}
    \[
    w(B_j^{t - 1}) = \sum_{c_u \in B_j^{t - 1}} p_u.
    \]
    
    \item \textbf{Scaling factor of a color block in cluster \( F_a^{t - 1} \):}
    \[
    h(B_j^{t - 1}) = \frac{c_u(F_a^{t - 1})}{p_u} \quad \text{for any } c_u \in B_j^{t - 1},
    \]
    which is well-defined since
    \[
    \frac{c_v(F_a^{t - 1})}{p_v} = \frac{c_u(F_a^{t - 1})}{p_u} \quad \forall c_u \neq c_v \in B_j^{t - 1}.
    \]
    \end{itemize}

    Create two sets $\cb$ and $\mb$.
    \[
        \cb = \{ (B_i^{t - 1}, B_{i + 1}^{t - 1}) \mid h(B_{i + 1}^{t - 1}) > h(B_i^{t - 1})\}
    \]
    
    \[
        \mb = \{ (B_i^{t - 1}, B_{i + 1}^{t - 1}) \mid h(B_{i + 1}^{t - 1}) \leq h(B_i^{t - 1})\}
    \]
    Informally, the set $\cb$ contains a pair of color blocks $B_i^{t - 1}$ and $B_{i + 1}^{t + 1}$ if the scaling factor of $B_{i + 1}^{t + 1}$ is greater than the scaling factor of $B_{i}^{t + 1}$. In the algorithm $\fmulti$, we cut the surplus for these types of pairs of color blocks. For the pair of color blocks in $\mb$ we merge the deficit.

    Now, for the pairs of color blocks $(B_{i}^{t + 1}, B_{i + 1}^{t + 1}) \in \cb$ let us define some notations w.r.t. a partition $X_j$.

    \begin{itemize}
    \item Surplus w.r.t. two corresponding colors $c_j$ and $c_{\hat{j}}$ in a partition $X_\ell$ is defined as

    \[ 
        \sigma^\ell_{c_j} = c_{\hat{j}}(X_\ell) - \frac{p_{\hat{j}}}{p_j} c_j(X_\ell)
    \]
    
    Here, $(c_j, c_{\hat{j}}) \in B^{t - 1}_{2k - 1} \times B^{t - 1}_{2k}$ where $(B^{t - 1}_{2k - 1}, B^{t - 1}_{2k}) \in \cb$.

    It is straightforward to see that,
    \[
        s^a(c_j, c_{\hat{j}}) \geq \sum_{\ell = 1}^s \sigma^\ell_{c_j}
    \]
    Similar to the proof of \cref{lem:fair-power-of-two} we can redefine $\sigma^\ell_{c_j}$ in such a way that
    \[
        s^a(c_j, c_{\hat{j}}) = \sum_{\ell = 1}^s \sigma^\ell_{c_j}
    \]
    \item Surplus w.r.t. two consecutive blocks $B^{i - 1}_{2k - 1}$ and $B^{i - 1}_{2k}$ in a partition $X_\ell$ is defined as
    \[
        S^k_{\ell} = \sum_{c_j \in B^{t - 1}_{2k - 1}} \sigma^\ell_{c_j}
    \]
    \item Surplus of the partition $X_\ell$ is
    \[
        S_\ell = \sum_{B^{t - 1}_{2k - 1} \in \cb} S^k_\ell
    \]
    It is straightforward to see that 
    \[
        T_a = \sum_{\ell = 1}^s S_\ell.
    \]
\end{itemize}
    Since, $\m{F}^{t^*}$ satisfies the $t$th fairness constraint of the algorithm $\fmulti$ we have for a cluster $F_{r_\ell}^{t^*} \in \m{F}^{t^*}$
    \[
        \dfrac{c_u(F_{r_\ell}^{t^*})}{p_{u}}= \dfrac{c_v(F_{r_\ell}^{t^*})}{p_{v}}\, \,  \forall c_u \in B^{t - 1}_{2k - 1}, c_v \in B^{t - 1}_{2k}
    \]
    recall $X_\ell \subseteq F_{r_\ell}^{t^*} \in \m{F}^{t^*}$.

    Let, $M_\ell \subseteq  F_{r_\ell}^{t^*}\setminus F^{t-1}_{a} $

    To maintain the $t$th fairness constraint of the algorithm $\fmulti$ we must have
    \[
        |M_\ell| \geq \sum_{v \in M_\ell} w(v) \geq S_\ell
    \]
    Reasons behind the above inequalities
    \begin{itemize}
        \item \textbf{$1$st inequality:} $w(v) \leq 1$.
        \item \textbf{$2$nd inequality:} This follows from the requirement that the cluster $F^{t^*}_{r_\ell}$ must satisfy the $t$th fairness constraint of the algorithm $\fmulti$. That is, for every pair $(c_j, c_{\hat{j}}) \in B^{t - 1}_{2k - 1} \times B^{t - 1}_{2k}$, we must ensure that after adding the set $M_\ell$ to $X_\ell$ in the cluster $F^{t^*}_{r_\ell}$, the following holds:
\[
    \sum_{v \in C_j(X_\ell \cup M_\ell)} w(v) = \sum_{v \in C_{\hat{j}}(X_\ell \cup M_\ell)} w(v)
\]
where $C_r(S)$ denotes the set of vertices of color $c_r$ in a subset $S \subseteq V$. In other words, the total weight of the vertices of color $c_j$ and $c_{\hat{j}}$ in the updated cluster must be equal, for every such pair. The surplus $ S_\ell$ precisely captures the imbalance in these weights in $X_\ell$, and hence the total weight added must be at least $S_\ell$ to restore this balance.
    \end{itemize}
    Let, us define $\costone(F_a^{t - 1})$ be the number of pairs $(u,v)$ such that
    \begin{enumerate}
        \item Either $u \in M_\ell$ and $v \in X_\ell$
        \item or $u \in X_m$ and $v \in X_\ell$ for $m \neq \ell$.
    \end{enumerate}
    Hence,
    \begin{align}
        &\costone(F_a^{t - 1}) \n \\
        &\geq \frac{1}{2}\sum_{\ell = 1}^s\left(|M_\ell|(|X_\ell|) + \sum_{m \neq \ell}|X_m|(|X_\ell|)\right) \n \\
        &\geq \frac{1}{2} \sum_{\ell = 1}^s \left(S_\ell(|X_\ell| - S_\ell) + \sum_{m \neq \ell}S_m(|X_\ell| - S_\ell)\right) \n \\
        &\geq \frac{1}{2} \sum_{\ell = 1}^s T_a (|X_\ell| - S_\ell) \n \\
        &\geq \frac{1}{2} T_a (|F_a^{t - 1}| - T_a) \label{eq:cost-one-multi} 
    \end{align}

    Similarly, now for a pair of color blocks $(B_{2k - 1}^{t - 1}, B_{2k}^{t + 1}) \in \mb$ and for a partition $X_\ell$, let us define
    \begin{itemize}
        \item $\hat{M}_\ell \subseteq V \setminus F_a^{t - 1}$ that serves the deficit for the color blocks $B_{2k - 1}^{t - 1}$ and $B_{2k}^{t - 1}$ in the cluster $F_{r_\ell}^{t^*}$.
        \item $G_m = \bigcup_{(B_{2k-1}^{t - 1}, B_{2k}^{t - 1}) \in \mb} F_{r_m}^{t^*} \cap B_{2k - 1}^{t - 1}(F_a^{t - 1})$ : it denotes the part of the set $B_{2k - 1}^{t - 1}(F_a^{t - 1})$ that lies in the cluster $F_{r_m}^{t^*} \in \m{F}^{t^*}$ for $m \neq \ell$
    \end{itemize}

    It is straightforward to see that 
    \[
        |\hat{M}_\ell| + \sum_{m \neq \ell} |G_m| \geq D_a
    \]
    Let, us define $\costtwo(F_a^{t - 1})$ be the number of pairs $(u,v)$ such that
    \begin{enumerate}
        \item Either $u \in \hat{M}_\ell$ and $v \in X_\ell$
        \item or $u \in G_m$ and $v \in X_\ell$ for $m \neq \ell$.
    \end{enumerate}
    Hence,
    \begin{align}
        \costtwo(F_a^{t - 1}) &\geq \frac{1}{2}\sum_{\ell = 1}^s|\hat{M}_\ell||X_\ell| \n \\
        &+ \sum_{m \neq \ell}|G_m||X_\ell| \n \\
        &\geq \frac{1}{2}\sum_{j = 1}^s D_a (|X_\ell| - |S_\ell|) \n \\
        &\geq \frac{1}{2} D_a (|F_a^{t - 1}| - T_a) \label{eq:cost-two-multi} 
    \end{align}
    Since, in $\costone(F_a^{t - 1})$ and $\costtwo(F_a^{t - 1})$ we count disjoint pairs. Hence, 
    \begin{align}
        \pay{F_a^{t - 1}} \geq \costone(F_a^{t - 1}) + \costtwo(F_a^{t - 1}) \label{eq:pay-single-cluster}
    \end{align}

    Now. by \cref{eq:main-multi}, \cref{eq:cost-one-multi}, \cref{eq:cost-two-multi} and \cref{eq:pay-single-cluster} we conclude the proof of \cref{clm:lower-bound-opt-multi-one}.
    
\end{proof}

   \begin{claim}\label{clm:lower-bound-opt-multi-two}
        \begin{align}
        \dist(\m{F}^{t - 1}, \m{F}^{t^*}) \geq &\sum_{F_a^{t - 1} \in \m{F}^{t - 1}} \frac{1}{2}T_a^2  \n
    \end{align}
\end{claim}

\begin{proof}
    Similar to the proof of the previous claim, we define $S_\ell$ as the surplus of a partition $X_\ell$ and $M_\ell \subseteq V \setminus F_a^{t - 1}$ be the set of points that are merged to $X_\ell$ to satisfy the $t$th fairness constraint of $\fmulti$.

    Let, us define $\pay{F_a^{t - 1}, X_\ell}$ be the number of pairs $(u,v)$ s.t.
    \begin{itemize}
        \item $u \in M_\ell$, $v \in X_\ell$
        \item $u \in X_m$, $v \in X_\ell$ for all $m \neq \ell$.
    \end{itemize}
    Now,
    \begin{align*}
        \dist(\m{F}^{t - 1}, \m{F}^{t^*}) \geq \frac{1}{2} \sum_{F_a^{t - 1} \in \m{F}^{t - 1}} \sum_{\ell = 1}^s \pay{F_a^{t - 1}, X_\ell}
    \end{align*}

    We multiply by $1/2$ to avoid overcounting of pairs. Note we can overcount a pair in the following situations
    \begin{enumerate}[(i)]
        \item Since we take the sum overall $F_a^{t - 1} \in \m{F}^{t - 1}$, we can overcount a pair $(u,v)$ if $u \in F_b^{t - 1}$ and $v \in F_a^{t - 1}$ where $b \neq a$ when considering the cluster $F_b^{t - 1}$ in the summation.
        \item Since we take the sum over all partitions $X_\ell$ of a cluster $F_a^{t - 1}$, we can overcount a pair $(u,v)$ if $u \in X_m$ and $v \in X_\ell$ where $m \neq \ell$ when we consider the partition $X_m$ in the summation.
    \end{enumerate}
    Now, we only need to show
    \[
        \sum_{\ell = 1}^s \pay{F_a^{t - 1}, X_\ell} \geq T_a^2
    \]
    \begin{align}
        \sum_{\ell = 1}^s \pay{F_a^{t - 1}, X_\ell} & \geq \sum_{\ell = 1}^s (|M_\ell| |X_\ell| + \sum_{m \neq \ell}|X_m||X_\ell|) \n \\
        &\geq \sum_{\ell = 1}^s (|M_\ell| S_\ell + \sum_{m \neq \ell}S_m S_\ell) \n \\
        &\geq \sum_{\ell = 1}^s (S_\ell^2 + \sum_{m \neq \ell}S_m S_\ell) \n \\
        &\geq T_a^2 \n
    \end{align}
\end{proof}

\begin{claim}\label{clm:lower-bound-opt-new}
    \[\dist(\m{F}^{t - 1}, \m{F}^{t^*}) \geq \sum_{F_a^{t - 1} \in \m{F}^{t - 1}} \frac{1}{2} D_a^2 \]
\end{claim}

\begin{proof}
    For a pair of color blocks $(B_{2k - 1}^{t - 1}, B_{2k}^{t + 1}) \in \mb$ and for a partition $X_\ell$, let us define
    \begin{itemize}
        \item $\hat{M}_\ell \subseteq V \setminus F_a^{t - 1}$ that serves the deficit for the color blocks $B_{2k - 1}^{t - 1}$ and $B_{2k}^{t - 1}$ in the cluster $F_{r_\ell}^{t^*}$.
        \item $G_m = \bigcup_{(B_{2k-1}^{t - 1}, B_{2k}^{t - 1}) \in \mb} F_{r_m}^{t^*} \cap B_{2k - 1}^{t - 1}(F_a^{t - 1})$ : it denotes the part of the set $B_{2k - 1}^{t - 1}(F_a^{t - 1})$ that lies in the cluster $F_{r_m}^{t^*} \in \m{F}^{t^*}$ for $m \neq \ell$
    \end{itemize}

    It is straightforward to see that 
    \[
        |\hat{M}_\ell| + \sum_{m \neq \ell} |G_m| \geq D_a
    \]
    Let, us define $\pay{F_a^{t - 1}, X_\ell}$ be the number of pairs $(u,v)$ s.t.
    \begin{enumerate}
        \item Either $u \in \hat{M}_\ell$ and $v \in X_\ell$
        \item or $u \in G_m$ and $v \in X_\ell$ for $m \neq \ell$.
    \end{enumerate}
    \begin{align*}
        \dist(\m{F}^{t - 1}, \m{F}^{t^*}) \geq \frac{1}{2} \sum_{F_a^{t - 1} \in \m{F}^{t - 1}} \sum_{\ell = 1}^s \pay{F_a^{t - 1}, X_\ell}
    \end{align*}

    We multiply by $1/2$ to avoid overcounting of pairs. Note, we can overcount a pair in the following situations
    \begin{enumerate}[(i)]
        \item Since we take the sum overall $F_a^{t - 1} \in \m{F}^{t - 1}$, we can overcount a pair $(u,v)$ if $u \in F_b^{t - 1}$ and $v \in F_a^{t - 1}$ where $b \neq a$ when considering the cluster $F_b^{t - 1}$ in the summation.
        \item Since we take the sum over all partitions $X_\ell$ of a cluster $F_a^{t - 1}$, we can overcount a pair $(u,v)$ if $u \in G_m$ and $v \in X_\ell$ where $m \neq \ell$ when we consider the partition $X_m$ in the summation. Note $G_m \subseteq X_m$.
    \end{enumerate}
    Now, we only need to show
    \[
        \sum_{\ell = 1}^s \pay{F_a^{t - 1}, X_\ell} \geq D_a^2
    \]
    \begin{align}
        \sum_{\ell = 1}^s \pay{F_a^{t - 1}, X_\ell} &\geq \sum_{\ell = 1}^s\left(|\hat{M}_\ell||X_\ell| + \sum_{m \neq \ell}|G_m||X_\ell|\right) \n \\
        &\geq \sum_{\ell = 1}^s D_a |X_\ell| \n \\
        &\geq  D_a |F_a^{t - 1}| \n \\ 
        &\geq  D_a^2\, \, \text{(\textbf{as}$|F_a^{t - 1}| \geq D_a$)} \n
    \end{align}
\end{proof}

    \begin{claim}\label{clm:upper-bound-algo-multi}
        \begin{align}
        \dist(\m{F}^{t - 1}, \m{F}^t) \leq &\sum_{F_a^{t - 1} \in \m{F}^{t - 1}}  T_a\left(\left|F_a^{t - 1}\right| - T_a \right) \n \\
        &+ D_a \left(\left|F_a^{t - 1}\right| - D_a \right) \n \\
        &+ \frac{1}{2}D_a^2 + \frac{1}{2}T_a^2 \n
    \end{align}
    \end{claim}

    \begin{proof}
        For each cluster $F_a^{t - 1} \in \m{F}^{t - 1}$ we cut the surplus many vertices, $T_a$ and merge deficit many vertices $D_a$ to it.

        Hence, in $\dist(\m{F}^{t - 1}, \m{F}^t)$ for a cluster $F_a^{t - 1} \in \m{F}^{t - 1}$ we count the cost of cutting and merging to it, which is
        \[
            T_a (|F_a^{t - 1}| - T_a) + D_a (|F_a^{t - 1}| - D_a)
        \]
        Again for a cluster $F_a^{t - 1} \in \m{F}^{t - 1}$ the deficit many vertices that we merge can come from several other clusters. A trivial upper bound on this cost is given by $(1/2) \cdot D_a^2$.

        The surplus many vertices $T_a$, that we cut from $F_a^{t - 1} \in \m{F}^{t - 1}$ can further get divided. A trivial upper bound on the cost of dividing $T_a$ many points is $(1/2)\cdot T_a^2$.

        Hence we get,

        \begin{align}
        \dist(\m{F}^{t - 1}, \m{F}^t) \leq &\sum_{F_a^{t - 1} \in \m{F}^{t - 1}}  T_a\left(\left|F_a^{t - 1}\right| - T_a \right) \n \\
        &+ D_a \left(\left|F_a^{t - 1}\right| - D_a \right) \n \\
        &+ \frac{1}{2}D_a^2 + \frac{1}{2}T_a^2 \n
        \end{align}
    \end{proof}
    
    Now we complete the proof of \cref{clm:each-step-4-approx}.
    \begin{proof}[\textbf{Proof of \cref{clm:each-step-4-approx}}:]
        By \cref{clm:upper-bound-algo-multi} we get,
        \begin{align}
        \dist(\m{F}^{t - 1}, \m{F}^t) \leq &\sum_{F_a^{t - 1} \in \m{F}^{t - 1}}  T_a\left(\left|F_a^{t - 1}\right| - T_a \right) \n \\
        &+ D_a \left(\left|F_a^{t - 1}\right| - D_a \right) \n \\
        &+ \frac{1}{2}D_a^2 + \frac{1}{2}T_a^2 \label{eq:upper-bound-multi-zero}
    \end{align}
        By \cref{clm:lower-bound-opt-multi-one} we get
        \begin{align}
            &\sum_{F_a^{t - 1} \in \m{F}^{t - 1}}  T_a\left(\left|F_a^{t - 1}\right| - T_a \right) \n \\
        &+ D_a \left(\left|F_a^{t - 1}\right| - D_a \right) \leq 4 \dist(\m{F}^{t - 1}, \m{F}^{t^*}) \label{eq:upper-bound-multi-one}
        \end{align}
        By \cref{clm:lower-bound-opt-new} we get
        \begin{align}
            \sum_{F_a^{t - 1} \in \m{F}^{i - 1}} \frac{1}{2}D_a^2 \leq \dist(\m{F}^{t - 1}, \m{F}^{t^*}) \label{eq:upper-bound-multi-two}
        \end{align}
        By \cref{clm:lower-bound-opt-multi-two} we get
        \begin{align}
            \sum_{F_a^{t - 1} \in \m{F}^{i - 1}} \frac{1}{2}T_a^2 \leq \dist(\m{F}^{t - 1}, \m{F}^{t^*}) \label{eq:upper-bound-multi-three}
        \end{align}
        Now, by combining \cref{eq:upper-bound-multi-zero}, \cref{eq:upper-bound-multi-one}, \cref{eq:upper-bound-multi-two} and \cref{eq:upper-bound-multi-three} we get,
        \[
            \dist(F^{t - 1}, F^{t}) \leq 6 \dist(F^{t - 1}, F^{t^*})
        \]
    \end{proof}
    We have shown earlier that \cref{clm:each-step-4-approx} implies \cref{lem:analyze-fmulti}. Thus, we complete the proof of \cref{lem:analyze-fmulti}.
\end{proof}

\begin{proof}[Proof of \cref{thm:equitheorem}]
    To prove this theorem, let us describe an algorithm $\fequi$. The algorithm $\fequi$ takes a clustering $\m{D}$ as input. Each vertex of the clustering is colored from the colors $\chi = \{ c_1, c_2, \ldots, c_z\}$ where $|\chi|$ is not a power of two. 
    \paragraph{Describing $\fequi$:} We divide $\chi$ into $\log |\chi|$ many disjoint color groups $G_1, G_2, \ldots, G_{\log |\chi|}$ such that the size of each group, $|G_j|$ where $j \in [\log |\chi|]$ is a power of two. We do it greedily according to the binary representation of $|\chi|$. Consider the binary representation of $|\chi|$. For each index $j \in [\log \lceil|\chi|\rceil]$, if the corresponding bit is 1, create a group of size $2^{j-1}$.


    We apply the algorithm $\fptwo$ to get an intermediate clustering $\m{I} = \{I_1, \ldots, I_k\}$ such that for each cluster $I_j \in \m{I}$ where $j \in [k]$ we have for any pair of colors $c_a, c_b \in G_\ell$ where $\ell \in [\log|\chi|]$, 
    
    \[c_a(I_j) = c_b(I_j)\]

    For a subset $S \subseteq V$, let us define $G_\ell(S)$ as the set of vertices $v \in S$ such that $v$ has a color in $G_\ell$. 
    Now, our goal is to get a clustering $\m{F} = \{ F_1, \ldots, F_s \}$ from $\m{D}$ such that for each cluster $F_k \in \m{F}$ we have $c_u(F_k) = c_v(F_k)$ where $c_u, c_v \in \chi$.

     Given the intermediate clustering \(\mathcal{I}\) as input, the algorithm \(\fmulti\) produces a final $\fair$ \(\mathcal{F} = \{ F_1, \ldots, F_s \}\) and thus for each cluster \( F_j \in \mathcal{F} \), the following proportion holds as stated in \cref{lem:analyze-fmulti}:

\begin{align*}
    &|G_1(F_j)| : |G_2(F_j)| : \ldots : |G_{\log |\chi|}(F_j)| \n \\
    &= |G_1| : |G_2| : \ldots : |G_{\log |\chi|}|
\end{align*}

To apply \cref{lem:analyze-fmulti}, we consider a group of colors $G_\ell$ as a single color $z_\ell$(say). Here the clustering $\m{I}$ is a $p$-divisible clustering because for any cluster $I_j \in \m{I}$ we have $|G_\ell(I_j)|$ is divisible by $|G_\ell|$.

Furthermore, the algorithm \(\fmulti\) ensures uniformity within each group: for every group \( G_{\ell} \) and for any pair of colors \( c_a, c_b \in G_{\ell} \) with \( \ell \in [\log |\chi|] \), it guarantees that \( c_a(F_j) = c_b(F_j) \). As a consequence, for any pair of colors \( c_u, c_v \in \chi \), we have \( c_u(F_j) = c_v(F_j) \), i.e., each color is equally represented within every cluster of \(\mathcal{F}\).

    By \cref{lem:analyze-fmulti} we get that, $\m{F}$ is $O(\log^{2.8} |\chi|)$-close to the clustering $\m{I}$. By \cref{lem:fair-power-of-two} we get that $\m{I}$ is $O(|\chi|^{1.6})$-close to the clustering $\m{D}$. By applying the triangle inequality to the two preceding results, we conclude that the output clustering $\m{F}$ produced by $\fequi$ is $O(|\chi|^{1.6} \log^{2.8} |\chi|)$-close to the input clustering $\m{D}$. 

    Let $\m{F}^*$ be the closest fair clustering to $\m{D}$. Hence, we get,

    \begin{align}
        \dist(\m{D}, \m{F}) &\leq \dist(\m{D}, \m{I}) + \dist(\m{I}, \m{F}) \, \, \text{(triangle inequality)} \n \\
        &\leq \dist(\m{D}, \m{I}) + O(\log^{2.8}|\chi|) \dist(\m{I}, \m{F}^*) \n \\
        &\leq O(|\chi|^{1.6}) \dist(\m{D}, \m{F}^*) + O(\log^{2.8}|\chi|) \dist(\m{I}, \m{F}^*) \n \\
        &\leq O(|\chi|^{1.6}) \dist(\m{D}, \m{F}^*) + O(\log^{2.8}|\chi|) (\dist(\m{D}, \m{I}) \n \\ &+ \dist(\m{D}, \m{F}^*)) \, \, \text{(triangle inequality)} \n \\
        &\leq O(|\chi|^{1.6} \log^{2.8}|\chi| + |\chi| + \log^{2.8}|\chi|) \dist(\m{D}, \m{F}^*) \n \\
        &\leq O(|\chi|^{1.6} \log^{2.8}|\chi|) \dist(\m{D}, \m{F}^*) \n
    \end{align}
    This completes the proof of \cref{thm:equitheorem}.
\end{proof}

\section{Arbitrary Proportion: Proof of \cref{lem:main-multiple-of-p}}

\begin{algorithm}[t]
\caption{$\pdca$}
\label{alg:p-divisible-clustering}
\KwIn{Clustering $\mathcal{D}$, colors $\chi = \{c_1, \ldots, c_t\}$, and ratios $p_1:p_2:\cdots:p_t$}
\KwOut{$p$-divisible clustering $\mathcal{M}$}

\ForEach{$c_j \in \chi$}{
    Create $\sigma_j / p_j$ empty clusters: $\texttt{extra\_clusters} = \{P_1, P_2, \ldots, P_{\sigma_j/p_j}\}$\;

    Initialize $\texttt{CUT} \gets \emptyset$, $\texttt{MERGE} \gets \emptyset$\;
    
    \ForEach{$D_i \in \mathcal{D}$}{
        \lIf{$|\sigma(D_i, c_j)| \leq p_j / 2$}{
            $\texttt{CUT} \gets \texttt{CUT} \cup \{D_i\}$
        }
        \lElse{
            $\texttt{MERGE} \gets \texttt{MERGE} \cup \{D_i\}$
        }
    }

    \While{$\texttt{CUT} \neq \emptyset$}{
        Pick and remove $D_k \in \texttt{CUT}$\;
        Remove surplus: $D_k \gets D_k \setminus \sigma(D_k, c_j)$\;

        \If{$\texttt{MERGE} \neq \emptyset$}{
            \While{$\sigma(D_k, c_j) \neq \emptyset$}{
                \ForEach{$D_\ell \in \texttt{MERGE}$}{
                    $T \gets \min(|\sigma(D_k, c_j)|, |\delta(D_\ell, c_j)|)$-sized subset of $\sigma(D_k, c_j)$\;
                    $D_\ell \gets D_\ell \cup T$\;
                    $\sigma(D_k, c_j) \gets \sigma(D_k, c_j) \setminus T$\;
                    \If{$c_j(D_\ell)$ is a multiple of $p_j$}{
                        $\texttt{MERGE} \gets \texttt{MERGE} \setminus \{D_\ell\}$\;
                    }
                }
            }
        }
        \Else{
            \While{$\sigma(D_k, c_j) \neq \emptyset$}{
                \ForEach{$P_m \in \texttt{extra\_clusters}$}{
                    $Q \gets$ subset of size $\min(p_j, |\sigma(D_k, c_j)|, p_j - |P_m|)$\;
                    $P_m \gets P_m \cup Q$\;
                    $\sigma(D_k, c_j) \gets \sigma(D_k, c_j) \setminus Q$\;
                    \If{$|P_m| = p_j$}{
                        $\texttt{extra\_clusters} \gets \texttt{extra\_clusters} \setminus \{P_m\}$\;
                    }
                }
            }
        }
    }

    \While{$\texttt{MERGE} \neq \emptyset$}{
        Pick $D_k \in \cut \cup \merge$ with minimum $\kappa^j(D_k) - \mu^j(D_k)$\;
        Remove surplus: $D_k \gets D_k \setminus \sigma(D_k, c_j)$\;
        \While{$\sigma(D_k, c_j) \neq \emptyset$}{
            \ForEach{$D_\ell \in \texttt{MERGE}$}{
                $T \gets \min(|\sigma(D_k, c_j)|, |\delta(D_\ell, c_j)|)$-sized subset\;
                $D_\ell \gets D_\ell \cup T$\;
                $\sigma(D_k, c_j) \gets \sigma(D_k, c_j) \setminus T$\;
                \If{$c_j(D_\ell)$ is a multiple of $p_j$}{
                    $\texttt{MERGE} \gets \texttt{MERGE} \setminus \{D_\ell\}$\;
                }
            }
        }
    }
}
\Return{$\mathcal{M}$ composed of updated $\mathcal{D}$ and filled $\texttt{extra\_clusters}$}\;
\end{algorithm}

\begin{proof}[Proof of \cref{lem:main-multiple-of-p}]
For a color $c_j \in \chi$, we created two sets $\cut$ and $\merge$ in the algorithm $\pdca$. Let us now define two cases.

\begin{itemize}

	\item Cut case for color $c_j$: $\cutcase$: 
        
	\[
		\text{If} \, \, \sum_{D_k \in \cut} |\sigma(D_k, c_j)| \geq \sum_{D_k \in \merge} |\mu(D_k, c_j)|
	\]
	\item Merge case for color $c_j$: $\mergecase$:
	\[
		\text{If} \, \, \sum_{D_k \in \cut} |\sigma(D_k, c_j)| < \sum_{D_k \in \merge} |\mu(D_k, c_j)|
	\]

\end{itemize}

For $\cutcase$, let us define some costs incurred by our algorithm $\pdca$

\begin{itemize}
     \item From each cluster $D_k \in \cut$, $\pdca$ cuts the surplus part $\sigma(D_k, c_j)$ from $D_k$. Hence, the cost paid for cutting these surplus points is the number of pairs $(u,v)$ such that $u \in \sigma(D_k, c_j)$ and $v \in (D_k \setminus \sigma(D_k, c_j))$. We denote this by:
     \begin{align}
         \costone(\m{M})^{c_j} =  \sum_{D_k \in \cut} |\sigma(D_k, c_j)| (|D_k| - |\sigma(D_k, c_j)|)\label{equn:cost-paid-one}
     \end{align}
    \item For each cluster $D_m \in \merge$, the algorithm $\pdca$  merges the deficit amount $|\delta(D_m, c_j)|$ to these clusters. Hence, the cost paid for merging the deficit to $D_m$ is the number of pairs $(u,v)$ such that $u \in \delta(D_m, c_j)$ and $v \in D_m$
    
     \begin{align}
         \costtwo(\m{M})^{c_j} =  \sum_{D_m \in \merge} |\delta(D_m, c_j)| |D_m| \label{equn:cost-paid-two}
     \end{align}
    
     \item The $\sigma(D_k, c_j)$ points that are cut from $D_k$ can get merged with the surplus of other clusters $\sigma(D_\ell, c_j)$ (say) which are also cut from a cluster $D_\ell \neq D_k$. We call the cost of merging $\sigma(D_k, c_j)$  with the surplus of other clusters as $\costthree(\m{M})^{c_j}$

     Let, $\sigma(D_k, c_j)$ gets further divided into multiple parts of size $\alpha_1, \ldots, \alpha_t$, so we get,
     \begin{align}
         \costthree(\m{M})^{c_j} \leq \sum_{D_k \in \m{D}} \dfrac{1}{2}\sum_{i = 1}^t\alpha_i (p_j - \alpha_i) \label{equn:cost-paid-three}
     \end{align}
    In the above expression we provide an upper bound on the number of pairs $(u,v)$ such that $u \in \sigma(D_k, c_x)$ and $v \in V \setminus D_k$ that are present together in a cluster in the clustering $\m{M}$. The above expression provides such an upper bound because of the fact that the surpluses which we cut from the clusters in the $\cut$ in our algorithm is used to fulfil the deficit of a cluster $D_m \in \merge$ where $m \neq k,\ell$ and we know $\delta(D_m, c_j) < p_j$. 
     
     \item These $\sigma(D_k, c_j)$ points from $D_k$ can also further be split into several parts $W_1, W_2, \ldots, W_t$ (say). These parts of $\sigma(D_k, c_j)$ points belong to different clusters in $\m{M}$ and thus would incur some cost. We call this cost as $\costfour(\m{M})^{c_j}$.
     \begin{align}
         \costfour(\m{M})^{c_j} =  \sum_{D_k \in \m{D}} \frac{1}{2} \sum_{s = 1}^t |W_s|(\sigma(D_k, c_j) - |W_s|)\label{equn:cost-paid-four}
     \end{align}
    In the above expression, we count the number of pairs $(u,v)$ such that $u,v \in \sigma(D_k, c_j)$ but present in different clusters in the clustering $\m{M}$.
\end{itemize}

For $\mergecase$, let us define some costs incurred by our algorithm $\pdca$.

\begin{itemize}
        \item In this case for a cluster $D_k \in \m{D}$, we may cut multiple subsets of size $p_j$ and a single subset of size $\sigma(D_k, c_j)$. Let us assume $Y_{k,z}$ denotes the $z$th such subset of the cluster $D_k$ and $y_{k,z}$ takes the value $1$ if we cut $z$th such subset from $D_k$. The cost of cutting $z$th such subset from $D_k$ is given as
     \begin{align}
         &\kappa_0(D_k) = |\sigma(D_k, c_j)| (|D_k| -  |\sigma(D_k, c_j)|) \n \\
         &(\text{cost of cutting the $0$th subset})\n \\
         &\kappa_z(D_k) = p_j (|D_k| -|\sigma(D_k, c_j)| - zp_j) \n \\
         &(\text{cost of cutting the $z$th subset for $z \geq 1$})\n 
    \end{align}
    Thus, we define 
    \begin{align}
         &\costfive(\m{M})^{c_j} = \sum_{D_k \in \m{D}} \sum_{z = 0}^t y_{k,z} \kappa_z(D_k) \label{equn:cost-paid-one-merge} \\
         &\text{where} \, \, \left( \text{t} = \frac{c_j(D_k) - |\sigma(D_k,c_j)|}{p_j} \right) \n
     \end{align}
     \item Suppose the algorithm $\pdca$ merges at a cluster $D_m \in \merge'$. Here, $\merge' \subseteq \merge$ denotes the set of clusters where the algorithm $\pdca$ has merged the deficit amount of points. More specifically, it is defined as
     \begin{align*}
         &D_m \in \merge' \iff \exists M_\ell \in \m{M} \n \\ 
         &\text{s.t.}\, \, D_m \subseteq M_\ell
     \end{align*}
        
    Then, the cost paid for merging the deficit to $D_m$ is
    \begin{align}
        \costsix(\m{M})^{c_j} = \sum_{D_m \in \merge'} \delta(D_m, c_j) |D_m| \label{equn:cost-paid-two-merge}
    \end{align}
    
     \item The $|Y_{k,z}|$ points that are cut from $D_k$ can get merged with the subsets $Y_{\ell,z'}$ of some other cluster $D_\ell$. We call the cost of merging a subset $Y_{k,z}$ of $D_k$ with the subset $Y_{\ell,z'}$ of another cluster $D_\ell$ as $\costseven(M)^{c_j}$.

     Let, $Y_{k,z}$ gets further divided into multiple parts of size $\alpha_1, \ldots, \alpha_t$, so we get, 
     \begin{align}
         \costseven(\m{M})^{c_j} \leq  \sum_{D_k \in \m{D}} \dfrac{1}{2}\sum_{i = 1}^t\alpha_i (p_j - \alpha_i) \label{equn:cost-paid-three-merge}
     \end{align}
     \item The $|Y_{k,z}|$ points that are cut from $D_k$ can also further be split into several parts $W_1, W_2, \ldots, W_t$ (say). These parts of $Y_{k,z}$ can belong to different clusters in $\m{M}$ (output of $\pdca$) and thus would incur some cost. We call this cost as $\costeight(M^{c_j})$.
     \begin{align}
         \costeight(\m{M})^{c_j} =  \sum_{D_k \in \m{D}} \frac{1}{2} \sum_{j = 1}^t |W_j|(|W_{i,z}| - |W_j|)
     \end{align}
\end{itemize}

There is a cost which can occur in both the $\cutcase$ and $\mergecase$.
\begin{itemize}
    \item Suppose for a cluster $D_k \in \m{D}$, deficit of $c_j$ and another color $c_r \in \chi$ is filled up by the subsets of some other clusters $D_\ell$ and $D_m$ in $\m{D}$ respectively such that $\ell \neq m$ then this would incur some cost which is the number of pairs $(u,v)$ such that $u \in \delta(D_k, c_j)$ and $v \in \delta(D_k, c_r)$.
     \begin{align}
         \costnine(\m{M})^{c_j} = \sum_{D_k \in D} |\delta(D_k, c_j)| |\delta(D_k, c_r)|
    \end{align}
\end{itemize}

Let us define $\pay(\m{M})^{c_j}$ as the number of pairs $(u,v)$ such that at least one of $u$ and $v$ is colored $c_j$ and the following conditions are true.
\begin{itemize}
    \item $u$ and $v$ are present in the same cluster in $\m{D}$ but in separate clusters in $\m{M}$.
    \item $u$ and $v$ are present in the separate clusters in $\m{D}$ but in the same cluster in $\m{M}$.
\end{itemize}

It is straightforward to see that,
\[
    \pay{\m{M}}^{c_j} \leq \sum_{i = 1}^9 \text{cost}_i(\m{M})^{c_j}
\]
and
\begin{align}
    \dist(\m{D}, \m{M}) = \sum_{c_j \in \chi} \pay{\m{M}}^{c_j} \label{eq:main-equation-multiple-of-p}
\end{align}
Now, to prove the above \cref{lem:main-multiple-of-p}, we take the help of the following claims from \cite{chakraborty2025towards}.

\begin{claim}\cite{chakraborty2025towards} \label{clm:one-chakraborty}
    $\costone(\m{M})^{c_j} + \costtwo(\m{M})^{c_j} + \costthree(\m{M})^{c_j} + \costfour(\m{M})^{c_j} \leq 3.5 \dist(\m{D}, \m{M}^*)$
\end{claim}

\begin{claim}\cite{chakraborty2025towards} \label{clm:two-chakraborty}
    $\costfive(\m{M})^{c_j} + \costsix(\m{M})^{c_j} + \costseven(\m{M})^{c_j} + \costeight(\m{M})^{c_j} \leq 3 \dist(\m{D}, \m{M}^*)$
\end{claim}

\begin{claim}\cite{chakraborty2025towards} \label{clm:three-chakraborty}
    $\costnine(\m{M})^{c_j} \leq \dist(\m{D}, \m{M}^*)$
\end{claim}

Now we complete the proof of \cref{lem:main-multiple-of-p}

    By \cref{clm:one-chakraborty}, \cref{clm:two-chakraborty}, \cref{clm:three-chakraborty} and \cref{eq:main-equation-multiple-of-p} we get
    \begin{align}
    \dist(\m{D}, \m{M}) &= \sum_{c_j \in \chi} 3.5 \dist(\m{D}, \m{M}^*) + 3 \dist(\m{D}, \m{M}^*) \n \\&+ \dist(\m{D}, \m{M}^*) \n \\
    &= \sum_{c_j \in \chi} 7.5 \dist(\m{D}, \m{M}^*) \n \\
    &= O(|\chi|) \dist(\m{D}, \m{M}^*) \n
\end{align}
\end{proof}

\section{Implication to Fair Correlation Clustering: Completing the Proof of \cref{lem:correlation_clustering}}
 Let us start by recalling the fair correlation clustering problem.
\paragraph{Fair Correlation Clustering.}
A clustering $\m{F}^*$ is called fair correlation clustering if given a correlation clustering instance $G$, $\cost{\m{F}^*}$ is minimum among all clusterings $\m{C}$ and it is also a $\fair$.


\paragraph{$\beta$-Approximate Fair Correlation Clustering.}
A fair clustering $\m{F}$ is called a \emph{$\beta$-approximate fair correlation clustering} if:
\[
\cost{\m{F}} \leq \beta \cdot \cost{\m{F}^*}.
\]

Given any arbitrary clustering $\m{C}$ let us construct its corresponding correlation clustering instance $G_{\m{C}}$ in the following way.
\begin{itemize}
    \item Let, $G_{\m{C}}$ be a complete graph that consists of the vertices in $\m{C}$.
    \item Each edge $(u,v)$ is labelled ``$+$'' if $u$ and $v$ are in the same cluster in $\m{C}$.
    \item $(u,v)$ is labelled ``$-$'' if $u$ and $v$ are in different clusters in $\m{C}$.
\end{itemize}
We know for a correlation clustering instance $G$, by definition for a clustering $\m{K}$ on $G$ we have
\begin{align*}
    \cost{\m{K}} = &\text{Total number of intercluster ``$+$'' and} \\ &\text{intracluster ``$-$'' edges}
\end{align*}
Let us now define, for two correlation clustering instances $G$ and $H$
\begin{align*}
    \dist(G,H) = &\text{Number of pairs $(u,v)$ that are labelled} \\
    &\text{``$+$'' in $G$ and ``$-$'' in $H$ or viceversa.}
\end{align*}
It is easy to see that for any clustering $\m{K}$,
\begin{align*}
    \cost{\m{K}} = \dist(G, G_{\m{K}}).
\end{align*}

Let $\m{F}^*$ be the optimal fair correlation clustering of $G$. We need to prove that

\begin{align*}
    \dist(G, G_{\m{F}}) \leq (\gamma + \beta + \gamma\beta) \dist(G, G_{\m{F}^*}).
\end{align*}

To do this, observe the following:
\begin{align}
\dist(G, G_{\mathcal{F}}) &\leq \dist(G, G_{\mathcal{D}}) + \dist(G_{\mathcal{D}}, G_{\mathcal{F}}) \label{eq:triangle-1} \\
&\leq \beta \cdot \dist(G, G_{\mathcal{F}^*}) + \dist(G_{\mathcal{D}}, G_{\mathcal{F}}) \label{eq:approx-B} \\
&\leq \beta \cdot \dist(G, G_{\mathcal{F}^*}) + \gamma \cdot \dist(G_{\mathcal{D}}, G_{\mathcal{D}^*}) \label{eq:alpha-close} \\
&\leq \beta \cdot \dist(G, G_{\mathcal{F}^*}) + \gamma \cdot \dist(G_{\mathcal{D}}, G_{\mathcal{F}^*}) \label{eq:optimal-fair} \\
&\leq \beta \cdot \dist(G, G_{\mathcal{F}^*}) + \gamma ( \dist(G, G_{\mathcal{F}^*}) \n \\ &+ \dist(G, G_{\mathcal{D}}) ) \label{eq:triangle-2} \\
&\leq \beta \cdot \dist(G, G_{\mathcal{F}^*}) + \gamma \cdot \dist(G, G_{\mathcal{F}^*})\n \\ &+ \gamma \beta \cdot \dist(G, G_{\mathcal{F}^*}) \label{eq:final-bound} \\
&= (\gamma + \beta + \gamma \beta) \cdot \dist(G, G_{\mathcal{F}^*}) \notag
\end{align}
where:
\begin{itemize}
    \item \eqref{eq:triangle-1} follows from the triangle inequality on distance between correlation instances,
    \item \eqref{eq:approx-B} uses the fact that \( \mathcal{D} \) is a \( \beta \)-approximation to the optimal fair clustering,
    \item \eqref{eq:alpha-close} uses that \( \mathcal{F} \) is \( \gamma \)-close to \( \mathcal{D} \),
    \item \eqref{eq:optimal-fair} uses that \( \mathcal{F}^* \) is a fair clustering,
    \item \eqref{eq:triangle-2} again applies the triangle inequality,
    \item \eqref{eq:final-bound} substitutes the bound from \eqref{eq:approx-B}.
\end{itemize}
This completes the proof.
\section{Implication to Fair Consensus Clustering}

In this section, we prove the existence of an algorithm that outputs $O(|\chi|^{1.6} \log^{2.8}|\chi|)$-approximate fair consensus clustering in the $(1:1:\cdots:1)$ case and an $O(|\chi|^{3.8})$-approximate fair consensus clustering in the general $(p_1:p_2:\cdots:p_{|\chi|})$ case.

Before that, let us formally define Consensus and Fair Consensus Clustering.

\begin{definition}[Consensus Clustering]
    Given a set of clusterings $\m{C}_1, \m{C}_2, \ldots, \m{C}_n$, a clustering $\m{C}^*$ is a consensus clustering if it minimizes the objective
    \[
        \left( \sum_{i = 1}^n \dist(C_i, \m{C}^*)^\ell\right)^{1/\ell}
    \]
    for any $\ell \in \mathbb{Z}^+$
\end{definition}

\begin{definition}[Fair Consensus Clustering]
    Given a set of clusterings $\m{C}_1, \m{C}_2, \ldots, \m{C}_n$, a clustering $\m{F}^*$ is a  fair consensus clustering if it minimizes the objective
    \[
        \left( \sum_{i = 1}^n \dist(C_i, \m{F}^*)^\ell\right)^{1/\ell}
    \]
    and also fair. Here $\ell$ is any positive integer.
\end{definition}

\begin{definition}[$\beta$-approximate Fair Consensus Clustering]
    A clustering $\m{F}$ is called a $\beta$-approximate Fair Consensus Clustering if the following is true
    \begin{align*}
        &\left( \sum_{i = 1}^n \dist(C_i, \m{F})^{\ell}\right)^{1/\ell} \n \\
        &\leq \beta \left( \sum_{i = 1}^n \dist(C_i, \m{F}^*)^{\ell}\right)^{1/\ell}
    \end{align*}
\end{definition}

Now we are ready to state the theorem

\begin{theorem}\label{thm:consensus-clustering}
    Given a set of points $V$, where each $v \in V$ has a color from the set $\chi$. There exists an algorithm that, given a set of clusterings $\m{C}_1, \ldots, \m{C}_m$, finds an $O(|\chi|^{1.6} \log^{2.8}|\chi|)$ approximate consensus fair clustering $\m{F}$ in the $(1:1:\cdots:1)$ case and an $O(|\chi|^{3.8})$ approximate consensus fair clustering in the general $(p_1:p_2:\cdots:p_{|\chi|})$ case in $O(m^2|V|^2)$ time. 
\end{theorem}

To prove the above theorem we need the help of the following lemma, which is implied from an algorithm given by \cite{chakraborty2025towards}.

\begin{lemma}\cite{chakraborty2025towards}\label{lem:consensus-chakraborty}
    Given a set of points $V$, where each point $v \in V$ has a color from the set $\chi$. Suppose there exists an algorithm that finds an $\alpha$-close fair clustering $\m{N}$ to a given clustering $\m{D}$ in $O(|V| \log |V|)$ time, then there exists an algorithm such that given $n$ input clusterings, it finds an $(\alpha + 2)$ approximate Fair Consensus Clustering in $O(m^2|V|^2)$ time. 
\end{lemma}

\begin{proof}[Proof of \cref{thm:consensus-clustering}]
    We know that given an input clustering $\m{D}$ the algorithm $\fequi$ outputs an $O(|\chi|^{1.6}\log^{2.8}|\chi|)$ close-fair clustering $\m{F}$ to $\m{D}$ in $(1:1:\cdots:1)$ case and the algorithm $\fgen$ outputs $O(|\chi|^{3.8})$ close-fair clustering $\m{F}$ to $\m{D}$ for arbitrary ratios.

    Hence, by using \cref{lem:consensus-chakraborty} we get that there exists an algorithm that finds a $O(|\chi|^{1.6} \log^{2.8}|\chi|)$ approximate fair consensus clustering in $(1:1:\cdots:1)$ case and $O(|\chi|^{3.8})$ approximate fair consensus clustering for arbitrary ratios in $O(m^2|V|^2)$ time.
\end{proof}

\section{Hardness: Proof of \cref{lem:yes.instance}}\label{app:hardness}
\yesinstance*
\begin{proof}
    It suffices to construct a fair clustering $ \hfairoptclss $ satisfying $ \dist(\hinpclss, \hfairoptclss ) = \tau $.

    Suppose $S$ is a $ \hyes $ instance of the $ \thrp $. Then there exists a partition $ S_1, S_2, \ldots, S_{n/3} $ of $ S = \{ x_1, x_2, \ldots, x_n\}$ such that for all $ 1\le i \le n/3 $, $ \card{S_{i}}= 3 $ and
    \begin{align}
        \sum_{x_j \in S_i}x_j = T \text{ where } T = \dfrac{\sum_{x_k \in S} x_k}{\frac{n}{3}}. \nonumber
    \end{align}
    Let, $ S_i = \{ x_{i_1}, x_{i_2}, x_{i_3}\} $. By our construction of $ (\hinpclss, \tau) $, we have $ \card{\hrcls{i_{j}}} = x_{i_j}$, for $j \in \{1,2,3\}$.

    $ \bullet $ If $ k=3 $, we construct $ \hfairoptclss $ by merging $ \hrcls{i_{j}}$ with $ \card{\hrcls{i_{j}}} $ points of color $ c_{2} $ and $ \card{\hrcls{i_{j}}} $ points of color $ c_{3} $ in $ \hgbcls{i} $ for $ j \in \{ 1, 2, 3\} $. More formally,
    \begin{align}
        \hfairoptclss = \left\{ \left(\hgbcls{i_{j}} \cup \hrcls{i_{j}} \right) \middle| i \in [n/3], j \in \{1, 2, 3\} \right \} .\nonumber
    \end{align}
    where $ \hgbcls{i_{j}} \subseteq \hgbcls{i} $ such that $ \hgbcls{i_{j}} $ consists of $ \card{\hrcls{i_{j}}} $ points of color $ c_{2} $ and $ \card{\hrcls{i_{j}}} $ points of color $ c_{3} $, for $\ \in \{1, 2, 3\} $.
    
    It can be seen that $ \hfairoptclss $ is a $\fair$ because for each cluster $ \hfaircls = (\hgbcls{i_{j}}\cup \hrcls{i_{j}}) \in \hfairoptclss $ we have $ \card{c_{1}( \hfaircls )} = \card{\hrcls{i_{j}}} = x_{i_{j}} $, $ \card{c_{2}( \hfaircls )} = \card{c_2(\hgbcls{i_{j}})} = x_{i_{j}} $, and $ \card{c_{3}( \hfaircls )} = \card{c_3(\hgbcls{i_{j}})} = x_{i_{j}} $.

    It remains to show that
    \begin{align}
        \dist(\hinpclss, \hfairoptclss) = 2\sum_{i=1}^{n}x_{i}^{2} + 2\sum_{i=1}^{n}x_{i}(T-x_{i}) = \tau. \nonumber
    \end{align}
    Indeed, for each cluster $ \hrcls{i_{j}} $, merging with $ 2\card{\hrcls{i_{j}}}=2x_{i_{j}} $ points from $ \hgbcls{i_{j}} $ costs $ 2x_{i_{j}}^{2} $. Summing this costs for all such clusters results in $ 2\sum_{i=1}^{n}x_{i}^{2} $. Finally, splitting each cluster $ \hgbcls{i} $ into three clusters $ \hgbcls{i_{j}} $ of size $ 2x_{i_{j}} $, for $ j\in \{1,2,3\} $ incurs a cost of $ \frac{1}{2}\sum_{j=1}^{3}2x_{i_{j}}(2T - 2x_{i_{j}}) $. Summing this costs for all clusters $ \hgbcls{i} $ results in $ 2 \sum_{i=1}^{n}x_{i}(T-x_{i}) $.

    Hence, $(\hinpclss,\tau) $ is a $ \hyes $ instance of the $ \threeclsf$.

    $ \bullet $ If $ k\geq 4 $, $ \hfairoptclss $ is constructed by merging each $ \hgbcls{i} $ with three clusters $ \hrcls{i_{1}}, \hrcls{i_{2}} $, and $ \hrcls{i_{3}} $. In other words,
    \begin{align}
        \hfairoptclss = \left\{ \hfaircls[i]=\left(\hgbcls{i}\cup \hrcls{i_{1}}\cup \hrcls{i_{2}}\cup \hrcls{i_{3}}\right)| i\in [n/3] \right\}. \nonumber
    \end{align}
    The fairness of $ \hfairoptclss $ is ensured since each cluster $ \hfaircls[i]\in \hfairoptclss $, $ \card{c_{j}(\hfaircls[i])} = \card{c_{j}(\hgbcls{i})} = T $, for $ 2\leq j\leq k $, and $ \card{c_{1}(\hfaircls[i])} = \card{\hrcls{i_{1}}} + \card{\hrcls{i_{2}}} + \card{ \hrcls{i_{3}} } = T $.

    The distance $ \dist(\hinpclss, \hfairoptclss ) $ consists of the following cost. For each $ \hgbcls{i} $, merging with $ T $ points of color $ c_{1} $ from $ \hrcls{i_{1}}, \hrcls{i_{2}}, \hrcls{i_{3}} $ incurs the cost $ \card{\hgbcls{i}}T = (k-1)T^{2} $. For each $ i=1,2, \dots, n/3 $, the cost of merging the three clusters $ \hrcls{i_{1}}, \hrcls{i_{2}}, \hrcls{i_{3}} $ together is $ \frac{1}{2} \sum_{j=1}^{3}x_{i_{j}}(T-x_{i_{j}}) $. Overall, we have
    \begin{align}
        \dist(\hinpclss, \hfairoptclss ) = \sum_{i=1}^{n/3}(k-1)T^{2} + \dfrac{1}{2}\sum_{i=1}^{n}x_{i}(T-x_{i}) = \tau.\nonumber
    \end{align}
    This concludes that $ (\hinpclss, \tau) $ is a $ \hyes $ instance of $ \clsf{k} $.
\end{proof}

Our proof for~\cref{lem:no.instance} utilizes the following result.
\begin{lemma}[{\cite[Lemma 45, Lemma 49]{chakraborty2025towards}}]\label{lem:no.instance.closest.p.fair}
    Given $ S=\{x_{1},x_{2}, \dots, x_{n}\} $ a $ \hno $ instance of $ \thrp $.

    Given an integer $ p\geq 2 $. Consider a clustering $ \hclspfinpclss = \{\hbcls{1}, \hbcls{2}, \dots, \hbcls{n/3}, \hrcls{1}, \hrcls{2}, \dots, \hrcls{n}\} $ over a set of red-blue colored points $ V' $ where the ratio between the total number of blue and red points is $ p $. In $ \hclspfinpclss $, each $ \hbcls{i} $ is a monochromatic blue cluster of size $ T $, and each $ \hrcls{i} $ is a monochromatic red cluster of size $ x_{i} $. Let $ \tau $ be defined as in our reduction with $ k=p+1 $, that is,
    \begin{align}
        \tau = \begin{cases}
            \dfrac{n}{3}\sum_{n=1}^{n/3}pT^{2} + \dfrac{1}{2}\sum_{i=1}^{n}x_{i}(T-x_{i}), &\text{ if }p\geq 3\\
            2\sum_{i=1}^{n}x_{i}^{2} + 2 \sum_{i=1}^{n}x_{i}(T-x_{i}), &\text{ if } p = 2
        \end{cases}.\nonumber
    \end{align}
    Then, for every $ \fair $ $ \hclspfoptclss $ over $ V' $, it must hold that $ \dist(\hclspfinpclss, \hclspfoptclss ) > \tau $.
\end{lemma}
\begin{remark}
    Our arguments can be extended to show that the problem of finding a closest fair clustering to a given clustering, under arbitrary color ratios, is also $\npc$. The only difference is that we need to adjust the definition of $ \tau $ in our reduction to account for the arbitrary color ratios. Specifically, if the input clustering $ \hinpclss $ has color ratios $ c_{1}(V) : c_{2}(V) : \dots : c_{k}(V) = p_{1}:p_{2}:\dots :p_{k} $, where $ 1\leq p_{1}\leq p_{2}\leq \dots \leq p_{k} $ are positive integers, then we set
    \begin{align}
        \tau = \dfrac{n}{3}\left(p_{2}+p_{3}+\dots + p_{k}\right) p_{1} T^{2} + \dfrac{1}{2}\sum_{i=1}^{n}p_{1}^{2}x_{i}\left(T-x_{i}\right), \nonumber
    \end{align}
    if $ \frac{p_{2}+p_{3}+\dots+p_{k}}{p_{1}} > 1+\sqrt{2} $, and
    \begin{align}
        \tau = &\sum_{i=1}^{n}p_{1}(p_{2}+p_{3}+\dots+p_{k})x_{i}^{2} \nonumber \\ 
               +&\dfrac{1}{2}\sum_{i=1}^{n}(p_{2}+p_{3}+\dots+p_{k})^{2}x_{i}(T-x_{i}), \nonumber
    \end{align}
    if $ \frac{p_{2}+p_{3}+\dots+p_{k}}{p_{1}} < 1+\sqrt{2} $.

    The correctness of this reduction follows analogously to the proofs of~\cref{lem:yes.instance} and~\cref{lem:no.instance}. We note that in the arguments establishing the mapping from $ \hno $ instance of $ \thrp $ to a $ \hno $ instance of $ \clsf{k} $, we employ a variant of~\cref{lem:no.instance.closest.p.fair}, in which the ratio between the number of blue and red points is $ p/q $, with $ p>q\geq 1 $ being positive integers. The case $ p/q > 1+\sqrt{2} $ is addressed in~\cite[Remarks 46]{chakraborty2025towards}, while the case $ p/q < 1+\sqrt{2} $ is handled in~\cite[Remarks 50]{chakraborty2025towards}. This separation explains the two different definitions of $ \tau $ in our reduction.
\end{remark}

\clearpage

\end{document}